\documentclass[12pt]{article}

\usepackage{amsmath,amssymb,graphicx,color}

\usepackage[margin=1in]{geometry}
\usepackage{times,natbib}
\usepackage{amsmath,amssymb,amsthm}
\usepackage[pdftex,colorlinks]{hyperref}

\newcommand{\lemref}[1]{Lemma~\ref{#1}}
\newcommand{\thmref}[1]{Theorem~\ref{#1}}
\renewcommand{\eqref}[1]{Equation~(\ref{#1})}
\newcommand{\secref}[1]{Section~\ref{#1}}

\newcommand{\defref}[1]{Definition~\ref{#1}}

\newtheorem{theorem}{Theorem}[section]
\newtheorem{lemma}[theorem]{Lemma}
\newtheorem{remark}[theorem]{Remark}
\newtheorem{example}[theorem]{Example}

\newtheorem{counter-example}[theorem]{Counter example}
\newtheorem{open question}[theorem]{Open question}
\newtheorem{corollary}[theorem]{Corollary}
\newtheorem{conjecture}[theorem]{Conjecture}
\newtheorem{definition}[theorem]{Definition}

\newtheorem{claim}{Claim}

\DeclareMathOperator*{\argmax}{argmax}

\newcommand{\ignore}[1]{}

\newcommand{\ca}{{\cal A}}

\newcommand{\cd}{{\cal D}}

\newcommand{\ch}{{\cal H}}

\newcommand{\cl}{{\cal L}}

\newcommand{\cp}{{\cal P}}

\newcommand{\cx}{{\cal X}}
\newcommand{\cy}{{\cal Y}}
\newcommand{\cz}{{\cal Z}}

\newcommand{\cn}{{\cal N}}
\newcommand{\cR}{{\cal R}}

\newcommand{\onefunc}{{\mathbb I}}

\newcommand{\reals}{{\mathbb R}}

\DeclareMathOperator*{\E}{\mathbb{E}}
\DeclareMathOperator*{\Err}{Err}
\DeclareMathOperator*{\Gain}{Gain}
\DeclareMathOperator*{\VC}{VC}
\newcommand{\bad}{\mathrm{bad}}
\newcommand{\good}{\mathrm{good}}
\newcommand{\pac}{\mathrm{PAC}}
\newcommand{\erm}{\mathrm{ERM}}
\newcommand{\range}{\operatorname{range}}
\newcommand{\ldim}{\operatorname{L-Dim}}
\newcommand{\bldim}{\operatorname{BL-Dim}}
\newcommand{\pbi}{\mathrm{PBI}}

\newcommand{\todoinline}[1]{{\bf \color{red} todo: #1}}

\title{Multiclass Learnability and the ERM principle}

\author{Amit Daniely\thanks{Dept. of Mathematics, The Hebrew University, Jerusalem, Israel}
\hspace{1cm}
Sivan Sabato\thanks{School of Computer Science and Engineering, The Hebrew University, Jerusalem, Israel}
\hspace{1cm}
Shai Ben-David\thanks{David R. Cheriton School of Computer Science, University of Waterloo, Waterloo, Ontario, Canada}
\hspace{1cm}
Shai Shalev-Shwartz\thanks{School of Computer Science and Engineering, The Hebrew University, Jerusalem, Israel}
}

\begin{document}

\maketitle

\begin{abstract}
  We study the sample
  complexity of multiclass prediction in several learning settings. For the PAC setting our analysis reveals a surprising phenomenon: In sharp contrast to binary classification, we show that there exist multiclass hypothesis classes for which
  some Empirical Risk Minimizers (ERM learners) have lower sample complexity than
  others. Furthermore, there are classes that are learnable by some
  ERM learners, while other ERM learners will fail to learn them. 
   We
  propose a principle for designing good ERM learners, and use this
  principle to prove tight bounds on the sample complexity of learning
  {\em symmetric} multiclass hypothesis classes---classes that
  are invariant under permutations of label names.
  We further provide a characterization of mistake and regret bounds for multiclass learning in the online setting and the bandit setting, using new generalizations of Littlestone's dimension.
\end{abstract}

\section{Introduction}
Multiclass prediction is the problem of classifying an object into one of several 
possible target classes. This task surfaces in
many domains. Common practical examples include document categorization, object recognition in
computer vision, and web advertisement.

The centrality of the multiclass learning problem has spurred the
development of various approaches for tackling this task. Most of these
approaches fall under the following general description: There is an instance 
domain $\cx$ and a set of possible class labels $\cy$. The goal of the learner is to
learn a mapping from instances to labels. The learner receives training examples, and outputs a predictor which belongs to some hypothesis class $\ch \subseteq \cy^\cx$, where $\cy^\cx$ is the set of all functions from $\cx$ to
$\cy$.
We study the
sample complexity of the task of learning $\ch$, namely, how many random training examples
are needed for learning an accurate predictor from $\ch$. This question has been
extensively studied and is quite well understood for the binary case
(i.e, where $|\cy|=2$). In contrast, as we shall see, existing
theory of the multiclass case is less complete.

In the first part of the paper we consider multiclass learning in the
classical PAC setting of \cite{Valiant84}. 
Since the 1970's, following Vapnik and Chervonenkis's seminal work
on binary classification \citep{VapnikCh71}, it was widely believed that excluding
trivialities, if a problem is at all learnable, then uniform
convergence holds, and the problem is also learnable by every Empirical Risk Minimizer (ERM learner). The
equivalence between learnability and uniform convergence has been
proved for binary classification and for regression problems
\citep{KearnsScSe94,BartlettLoWi1996,AlonBeCsHa97}.  Recently,
\cite{ShalevShSrSr10} have shown that in the general
setting of learning of \citet{Vapnik95},
learnability is not equivalent to uniform convergence. Moreover,
some learning problems are learnable, but not with every
ERM. In particular, this was shown for an unsupervised learning
problem in the class of stochastic convex learning problems. The
conclusion in \cite{ShalevShSrSr10} is that the conditions for
learnability in the general setting are significantly more complex
than in supervised learning.  In this work we show that even in multiclass learning,
uniform convergence is not equivalent to learnability. We find this result surprising, since multiclass
prediction is very similar to binary classification.

This result raises once more the question of determining the true sample
complexity of multiclass learning, and the optimal learning algorithm in this setting. We provide conditions under which tight characterization of the sample complexity of a multiclass hypothesis class can be provided. Specifically, we consider the important case of hypothesis classes which are invariant to renaming of class labels. We term such classes \emph{symmetric} hypothesis classes. We show that the sample complexity for symmetric classes is tightly characterized
by a known combinatorial measure called the Natarajan dimension. We 
conjecture that this result holds for non-symmetric classes as well.

We further study multiclass sample complexity in other learning models. 
Overall, we consider the following categorization of learning models:
\begin{itemize}
\item Interaction with the data source (batch vs. online protocols):
  In the batch protocol, we assume that the training data is generated
  i.i.d. by some distribution $\cd$ over $\cx \times \cy$. The goal is
  to find, with a high probability over the training samples, a
  predictor $h$ such that $\Pr_{(x,y)\sim \cd}(h(x)\ne y)$ is as small as possible. In the online protocol  we receive examples one by one, and are asked to predict the label of each given example on the fly. Our goal is to make as few prediction mistakes as possible in the worst case (see \citealt{Littlestone87}).
\item The type of feedback (full information vs. bandits): In the full
  information setting, we receive the correct label of every example. 
  In the bandit setting, the learner first sees an unlabeled
  example, and then outputs its prediction for the label. Then, a binary feedback is received, indicating only whether the prediction was correct or not, but not
  revealing the correct label in the case of a wrong guess (see for example \citealt{AuerCeFrSc03,AuerCeFi02,KakadeShTe08}).
\end{itemize}
The batch/full-information model is the standard PAC setting, while the online/full-information model is the usual online setting. The online/bandits model is the usual multiclass-bandits setting. We are not aware of a treatment of the batch/bandit model in previous works.

\subsection*{Paper Overview}
After presenting formal definitions and notations in \secref{sec:notation},
we begin our investigation of multiclass sample complexity in the classical PAC learning setting.
Previous results have provided upper and lower bounds on the sample complexity of multiclass learning in this setting when using any ERM algorithm. The lower bounds are controlled by the \emph{Natarajan dimension}, a combinatorial measure which generalizes the VC dimension for the multiclass case, while the upper bounds are controlled by the \emph{graph dimension}, which is another generalization of the VC dimension. The ratio between these two measures can be as large as $\Theta(\ln(k))$, where $k = |\cy|$ is the number of class labels. In \secref{sec:previous} we survey known results, and also present a new improvement for the upper bound in the realizable case.
All the bounds here are uniform, that is, they hold for all ERM learners. 

These uniform bounds are the departure point of our research. Our goal is to find a combinatorial measure, similar to the VC-Dimension, that characterizes the sample complexity of a given class, up to logarithmic factors, {\em independent of the number of classes}. We delve into this challenge in \secref{sec:pac}. First, we show that no uniform bound on arbitrary ERM learners can tightly characterize the sample complexity: We describe a family of concept classes for which there
exist `good' ERM learners and `bad' ERM learners, with a ratio of $\Theta(\ln(k))$ between their sample complexities. We further show that if $k$ is infinite, then
there are hypothesis classes that are learnable by some ERM learners but not by other ERM learners. Moreover, we show that for any hypothesis class, the sample complexity of the {\em worst} ERM learner in the realizable case is characterized by the graph dimension. 

These results indicate that classical concepts which are commonly used to provide upper bounds for all ERM learners of some hypothesis class, such as the growth function, cannot lead to tight sample complexity characterization for the multiclass case. We thus propose algorithmic-dependent versions of these quantities, that allow bounding the sample complexity of specific ERM learners.

We consider three cases in which we show that the true sample complexity of multiclass learning in the PAC setting is fully characterized by the Natarajan dimension. The first case includes any ERM algorithm that does not use too many class labels, in a precise sense that we define via the new notion of \emph{essential range} of an algorithm. In particular, the requirement is satisfied by any ERM learner which only predicts labels that appeared in the sample. The second case includes any ERM learner for symmetric hypothesis classes. The third case is the scenario where we have no prior knowledge on the different class labels, which we defined precisely in \secref{sec:symmetrize}.

We conjecture that the upper bound obtained for symmetric classes holds for non-symmetric classes as well. Such a result cannot be
implied by uniform convergence alone, since, by the results mentioned above,  there always exist ERM learners with a sample complexity that is higher than this conjectured upper bound. It therefore follows that a proof of our conjecture will require the derivation of new learning rules. We hope that this would lead to new insights in other statistical learning problems as well.

In \secref{sec:other} we study multiclass learnability in the online model and in the bandit model. We introduce two generalizations of the Littlestone dimension, which characterize multiclass learnability in each of these models respectively. Our bounds are tight for the realizable case.

\section{Problem Setting and Notation}\label{sec:notation}
Let $\cx$ be a space, $\cy$ a discrete space\footnote{To avoid measurability issues, we assume that $\cx$ and $\cy$ are countable.} and $\ch$ a class of functions from $\cx$ to $\cy$. Denote $k = |\cy|$ (note that $k$ can be infinite).
For a distribution $\cd$ over $\cx \times \cy$, the error of a function $f:\cx\to\cy$ with respect to $\cd$ is defined as
$\Err(f)=\Err_{\cd}(f)=\Pr_{(x,y)\sim \cd}(f(x)\ne y)$.
The best error achievable by $\ch$ on $\cd$, namely, $\Err_\cd(\ch):=\inf_{f\in\ch}\Err_{\cd}(f)$, is called the \emph{approximation error} of $\ch$ on $\cd$.

In the PAC setting, a \emph{learning algorithm} for a class $\ch$ is a function, $\ca:\cup_{n=0}^\infty (\cx\times
\cy)^n \to \cy^{\cx}$.  We denote a training sequence by $S_m =
\{(x_1,y_1),\ldots,(x_m,y_m)\}$.  An \emph{ERM learner} for class $\ch$ is a learning algorithm that for any sample $S_m$ returns a function that minimizes the empirical error relative to any other function in $\ch$. Formally, the empirical error of a function $f$ on a sample $S_m$ is 
\[
\Err_{S_m}(f)=\frac 1m|\{i \in [m] : f(x_i)\ne y_i\}|.
\]
A learning algorithm $\ca$ of class $\ch$ is an ERM learner if $\Err_{S_m}(\ca(S_m)) = \min_{f \in \ch} \Err_{S_m}(f)$.

The \emph{agnostic sample complexity} of a learning algorithm $\ca$ is the function $m_{\ca,\ch}^a$ defined as follows: For every $\epsilon,\delta>0$, $m^a_{\ca,\ch}(\epsilon,\delta)$ is the minimal integer such that for every $m\ge m^a_{\ca,\ch}(\epsilon,\delta)$ and every distribution $\cd$ on $\cx\times \cy$,
\begin{equation}\label{eq:samplecomplexity}
\Pr_{S_m\sim \cd^m}\left(\Err_{\cd}(\ca(S_m))> \Err_{\cd}(\ch)+\epsilon\right)\le\delta.
\end{equation}
Here and in subsequent definitions, we omit the subscript $\ch$ when it is clear from context. 
If there is no integer satisfying the inequality above, define
$m^a_{\ca}(\epsilon,\delta)=\infty$. $\ch$ is learnable with $\ca$ if for all $\epsilon$ and $\delta$ the agnostic sample complexity is finite. The agnostic sample complexity of a class $\ch$ is
\[
m^a_{\pac,\ch}(\epsilon,\delta)=
\inf_{\ca}m^a_{\ca,\ch}(\epsilon,\delta) ~,
\]
where the infimum is taken over all learning algorithms for $\ch$. 
The {\em agnostic ERM sample complexity} of $\ch$ is the sample complexity that can be guaranteed for any ERM learner. It is defined by
\[
m^a_{\erm,\ch}(\epsilon,\delta)=
\sup_{\ca \in \erm}m^a_{\ca,\ch}(\epsilon,\delta) ~,
\]
where the supremum is taken over all ERM learners for $\ch$. Note that always $m_{\pac}\le m_{\erm}$.

We say that a distribution
$\cd$ is \emph{realizable} by a hypothesis class $\ch$ if there exists some $f\in \ch$
such that $\Err_{\cd}(f)=0$.  The \emph{realizable sample complexity} of an algorithm
$\ca$ for a class $\ch$, denoted $m^r_\ca$, is the minimal integer such that for
every $m\ge m^r_\ca(\epsilon,\delta)$ and every distribution $\cd$ on $\cx\times
\cy$ which is realizable by $\ch$, \eqref{eq:samplecomplexity}
holds. The realizable sample complexity of a class $\ch$ is
$m^r_{\pac,\ch}(\epsilon,\delta)=\inf_{\ca}m^r_{\ca}(\epsilon,\delta)$, where the infimum is taken
over all learning algorithms for $\ch$. The realizable ERM sample complexity of a class $\ch$ is
$m^r_{\erm,\ch}(\epsilon,\delta)=\sup_{\ca\in \erm}m^r_{\ca}(\epsilon,\delta)$, where the supremum is taken over all ERM learners for $\ch$. 

Given a subset $S \subseteq \cx$, we denote
$\ch|_S=\{f|_S:f\in\ch\}$, where $f|_S$ is the restriction of $f$ to
$S$, namely, $f|_S : S \to \cy$ is such that for
all $x \in S$, $f|_S(x)=f(x)$.

\section{Uniform Sample Complexity Bounds for ERM Learners}\label{sec:previous}

We first recall some known results regarding the sample complexity of
multiclass learning.  Recall the definition of the Vapnik-Chervonenkis
dimension \citep{Vapnik95}:
\begin{definition}[VC dimension]
Let $\ch\subseteq \{0,1\}^\cx$ be a hypothesis class. A subset $S\subseteq \cx$ is \emph{shattered} by $\ch$ if $\ch|_S=\{0,1\}^S$. The \emph{VC-dimension} of $\ch$, denoted $\VC(\ch)$, is the maximal cardinality of a subset $S\subseteq \cx$ that is shattered by $\ch$.
\end{definition}
The VC-dimension, a cornerstone in statistical learning theory, characterizes the sample complexity of learning {\em binary} hypothesis classes, as the following bounds suggest.
\begin{theorem}[\citealp{Vapnik95} and \citealp{BartlettMe02}]\label{th:binary-case}
There are absolute constants $C_1,C_2>0$ such that for every $\ch\subseteq \{0,1\}^\cx$, 
\[
C_1\left(\frac{\VC(\ch)+\ln(\frac{1}{\delta})}{\epsilon}\right)\le
m^r_{\pac}(\epsilon,\delta)\le m^r_{\erm}(\epsilon,\delta)\le C_2\left(\frac{\VC(\ch)\ln(\frac{1}{\epsilon})+\ln(\frac{1}{\delta})}{\epsilon}\right),
\]
and
\[
C_1\left(\frac{\VC(\ch)+\ln(\frac{1}{\delta})}{\epsilon^2}\right) \leq 
m_{\pac}^a(\epsilon,\delta) \leq m_{\erm}^a(\epsilon,\delta) \leq C_2\left(\frac{\VC(\ch)+\ln(\frac{1}{\delta})}{\epsilon^2}\right).
\]
\end{theorem}
One of the important implications of this result is that in binary classification, \emph{all} ERM learners are as good, up to a multiplicative factor of $\ln(1/\epsilon)$.

It is natural to seek a generalization of the VC-dimension to
hypothesis classes of non-binary functions. We recall two generalizations, both introduced by \cite{Natarajan89b}. In
both generalizations, shattering of a set $S$ is redefined by requiring that for any partition of $S$ into $T$ and $S \setminus T$, there exists a $g \in \ch$ whose behavior on $T$ differs from its behavior on $S \setminus T$.
The two definitions are distinguished by their definition of ``different behavior''.
\begin{definition}[Graph dimension]
Let $\ch\subseteq \cy^\cx$ be a hypothesis class and let $S\subseteq \cx$. We say that $\ch$ \emph{G-shatters} $S$ if there exists an $f:S\rightarrow \cy$ such that for every $T\subseteq S$ there is a $g\in \ch$ such that
\[
\forall x\in T,\: g(x)=f(x),\text{ and \:}\forall x\in S\setminus T,\: g(x)\ne f(x).
\]
The \emph{graph dimension} of $\ch$, denoted $d_G(\ch)$, is the maximal cardinality of a set that is G-shattered by $\ch$. 
\end{definition}

\begin{definition}[Natarajan dimension]
Let $\ch\subseteq \cy^\cx$ be a hypothesis class and let $S\subseteq \cx$. 
We say that $\ch$ \emph{N-shatters} $S$ if there exist $f_1,f_2: S \rightarrow \cy$ such that $\forall y\in S,\; f_1(y)\ne f_2(y)$, and for every $T\subseteq S$ there is a $g\in \ch$ such that
\[
\forall x\in T,\: g(x)=f_1(x),\text{ and \:}\forall x\in S\setminus T,\: g(x)=f_2(x).
\]
The \emph{Natarajan dimension} of $\ch$, denoted $d_N(\ch)$, is the maximal cardinality of a set that is N-shattered by $\ch$.
\end{definition}

Both of these dimensions coincide with the VC-dimension for
$k=2$. Note also that we always have $d_N\le d_G$.
By reductions to and from the binary case, similarly to \cite{Natarajan89b} and \cite{Ben-DavidCeHaLo95} one can show the following result: 
\begin{theorem}\label{th:multiclaas-simple-bounds}
For the constants $C_1,C_2$ from Theorem \ref{th:binary-case}, for every $\ch\subseteq \cy^\cx$ we have
\[
C_1\left(\frac{d_N(\ch)+\ln(\frac{1}{\delta})}{\epsilon}\right)\le
m^r_{\pac}(\epsilon,\delta)\le m^r_{\erm}(\epsilon,\delta)\le
C_2\left(\frac{d_G(\ch)\ln(\frac{1}{\epsilon})+\ln(\frac{1}{\delta})}{\epsilon}\right),
\]
and 
\[
C_1\left(\frac{d_N(\ch)+\ln(\frac{1}{\delta})}{\epsilon^2}\right)
 \leq m_{\pac}^a(\epsilon,\delta) \leq m_{\erm}^a(\epsilon,\delta) \leq C_2\left(\frac{d_G(\ch)+\ln(\frac{1}{\delta})}{\epsilon^2}\right).
\]
\end{theorem}
\begin{proof} (sketch)
For the lower bound, let $\ch\subseteq \cy^\cx$ be a hypothesis class of Natarajan dimension $d$ and  Let $\ch_d:=\{0,1\}^{[d]}$. We claim that $m^r_{\pac,\ch_d}\le m^r_{\pac,\ch}$, and similarly for the agnostic sample complexity, so the lower bounds are obtained by Theorem \ref{th:binary-case}. Let $\ca$ be a learning algorithm for $\ch$. Consider the learning algorithm, $\bar \ca$, for $\ch_d$ defined as follows. Let $S=\{s_1,\ldots,s_d\}\subseteq X$ be a set and let $f_0,f_1$ be functions that witness the $N$-shattering of $\ch$. Given a sample $((x_i,y_i))_{i=1}^m \subseteq [d]\times \{0,1\}$, let $g=\ca((s_{x_i},f_{y_i}(s_{x_i}))_{i=1}^m)$. $\bar \ca$ returns $f:[d] \rightarrow \{0,1\}$ such that $f(i)=1$ if and only if $g(s_i)=f_1(s_i)$. It is not hard to see that $m^r_{\bar \ca,\ch_d}\le m^r_{\ca,\ch}$, thus $m^r_{\pac,\ch_d}\le m^r_{\pac,\ch}$ and similarly for the agnostic case.

For the upper bound, let $\ch\subseteq \cy^\cx$ be a hypothesis class of graph dimension $d$. For every $f\in \ch$ define $\bar f:\cx\times \cy\to \{0,1\}$ by setting $\bar f(x,y)=1$ if and only if $f(x)=y$ and let $\bar\ch =\{\bar f:f\in\ch\}$. It is not hard to see that $\VC(\bar \ch)=d_G(\ch)$.
Let $\ca$ be an ERM algorithm for $\ch$. Let $\bar \ca$ be an ERM algorithm for $\bar\ch$ such that for a sample $(((x_i,z_i),y_i))_{i=1}^m \subseteq \cx\times \cy \times \{0,1\}$, 
if for all $i$, $y_i = 1$, $\bar \ca$ returns $\bar{f}$, where $f = \ca((x_i,z_i)_{i=1}^m)$. It is easy to check that $\bar \ca$ is consistent and therefore can be extended to an ERM learner for $\bar \ch$, and that $m^r_{\ca,\ch} \leq m^r_{\bar \ca,\bar \ch}$. Thus $m^r_{\erm,\ch} \leq m^r_{\erm,\bar \ch}$. The analogous inequalities hold for the agnostic sample complexity as well. Thus the desired upper bounds follow from \thmref{th:binary-case}.
\end{proof}

This theorem shows that the finiteness
of the Natarajan dimension is a necessary condition for learnability, and the
finiteness of the graph dimension is a sufficient condition for
learnability. In \citet{Ben-DavidCeHaLo95} it was proved that for every hypotheses class $\ch\subseteq \cy^\cx$,
\begin{equation}\label{eq:ben-david}
d_N(\ch)\le d_G(\ch)\le 4.67\log_2(k)d_N(\ch) ~.
\end{equation}
It follows that if $k < \infty$ then the finiteness of the Natarajan
dimension is both a necessary and a sufficient condition for
learnability.\footnote{The result of \citet{Ben-DavidCeHaLo95} in fact holds also for a rich family of generalizations of the VC dimension, of which the Graph dimension is one example.} Incorporating \eqref{eq:ben-david} into \thmref{th:multiclaas-simple-bounds}, it can be seen that the Natarajan dimension,
as well as the graph dimension, characterize the sample complexity of
$\ch\subseteq \cy^\cx$ up to a multiplicative factor of
$O(\ln(k)\ln(\frac{1}{\epsilon}))$. Precisely, the following result can be derived:
\begin{theorem}\label{th:multiclaas-bounds}
There are constants $C_1,C_2$ such that, for every $\ch\subseteq \cy^\cx$,
\[
C_1\left(\frac{d_N(\ch)+\ln(\frac{1}{\delta})}{\epsilon}\right)\le
m^r_{\pac}(\epsilon,\delta)\le m^r_{\erm}(\epsilon,\delta)\le C_2\left(\frac{d_N(\ch)\ln(k)\cdot \ln(\frac{1}{\epsilon})+\ln(\frac{1}{\delta})}{\epsilon}\right),
\]
and
\[
C_1\left(\frac{d_N(\ch)+\ln(\frac{1}{\delta})}{\epsilon^2}\right)
 \leq m_{\pac}^a(\epsilon,\delta) \leq m_{\erm}^a(\epsilon,\delta) \leq C_2\left(\frac{d_N(\ch)\ln(k)+\ln(\frac{1}{\delta})}{\epsilon^2}\right).
\]
\end{theorem}

\subsection{An Improved Upper Bound for the Realizable Case}\label{sec:realizable}
The following theorem provides a sample complexity upper bound which provides a tighter dependence on $\epsilon$. 

\begin{theorem}\label{th:multiclass-bounds-our}
For every concept class $\ch\subseteq \cy^\cx$,
\[
m^r_{\erm}(\epsilon,\delta)=O\left(\frac{d_N(\ch)\left(\ln(\frac{1}{\epsilon})+\ln(k)+\ln(d_N(\ch))\right)+\ln(\frac{1}{\delta})}{\epsilon}\right). 
\]
\end{theorem}
The proof of this theorem is immediate given \thmref{thm:restricted_range}, which is provided in  \secref{sec:pac}.
We give the short proof of this theorem thereafter. 
While a proof for the Theorem can be established  by a simple
adaptation of previous techniques, we find it valuable to present this result
here, as we could not find it in the literature.

\section{PAC Sample Complexity with ERM Learners}\label{sec:pac}
In this section we study the sample complexity of multiclass ERM learners in the PAC setting. First, we show that unlike the binary case, in the multiclass setting different ERM learners can have very different sample complexities.

\begin{example} [A Large Gap Between ERM Learners] \label{ex:baderm}
Let $\cx$ be any finite or countable domain set. Let $\cp_{f}(\cx)$ denote the collection of finite and co-finite subsets $A\subseteq \cx$. We will take the label space to be $\cp_f(\cx)$ together with a special label, denoted by $*$ (I.e.
$\cy=\cp_{f}(\cx)\cup\{*\}$).
For every $A\in \cp_f(\cx)$, define $f_{A}:\cx\to \cy$ by
\[
f_A(x)=\begin{cases}
A & \text{if $x\in A$}\\
* & \text{otherwise},
\end{cases}
\]
and consider the hypothesis class $\ch_\cx=\{f_A:A \in \cp_f(\cx)\}$. It is not hard to see that $d_N(\ch_\cx) = 1$.
On the other hand, if $\cx$ is finite then $\cx$ is G-shattered using the function $f_\emptyset$, therefore $d_G(\ch_\cx) = |\cx|$. If $\cx$ is infinite, then every finite subset of $\cx$ is G-shattered, thus $d_G(\ch_\cx) = \infty$.

Consider two ERM algorithms for $\ch_\cx$, $\ca_\bad$ and $\ca_\good$, which satisfy the following properties. 
For $\ca_\bad$, whenever a sample of the form $S_m=\{(x_1,*),\ldots,(x_m,*)\}$ is observed,
$\ca_\bad$ returns $f_{\{x_1,\ldots,x_m\}^c}$. Intuitively, while $\ca_\bad$ selects a hypothesis that minimizes the empirical error, its choice for $S_m$ seems to be sub-optimal. We will show later, based on \thmref{th:ERM-Bounds}, that the sample complexity of ${\ca}_{\bad}$ is $\Omega\left(\frac{|\cx|+\ln(\frac{1}{\delta})}{\epsilon}\right)$.

For $\ca_\good$, we require that the algorithm only ever returns either $f_\emptyset$, or a hypothesis $A$ such that the label $A$ appeared in the sample---One can easily verify that there exists an ERM algorithm that satisfies this condition. Specifically, this means that for the sample $S_m=\{(x_1,*),\ldots,(x_m,*)\}$, $\ca_\good$ necessarily returns $f_\emptyset$. We have the following guarantee for $\ca_\good$:
\begin{claim} $m^r_{\ca_\good,\ch_\cx}(\epsilon,\delta) \leq \frac{1}{\epsilon}\ln{\frac{1}{\delta}}$, and  $m^a_{\ca_\good,\ch_\cx}(\epsilon,\delta) \leq \frac{1}{\epsilon^2}\ln(\frac{1}{\epsilon})\ln{\frac{1}{\delta}}$.
\end{claim}
\begin{proof}
We prove the bound for the realizable case. The bound for the agnostic case will be immediate using Cor.\ \ref{cor:only_observed_labels}, which we prove later. 

Let $\cd$ be a distribution over $\cx \times \cy$ and suppose that the correct labeling for $\cd$ is
$f_A$. Let $m$ be the size of the sample. For any sample, $\ca_{\good}$ returns
either $f_\emptyset$ or $f_A$. If it returns $f_A$ then its error on $\cd$ is zero. On the other hand, $\Err_\cd(f_\emptyset)=\Pr_{(X,Y)\sim \cd} (X \in A)$. Thus, $\ca_{\good}$ returns a hypothesis with error $\epsilon$ or more only if $\Pr_{(X,Y)\sim \cd} (X \in A)\ge\epsilon$ and all the $m$ examples in the sample
are from $A^c$. Assume $m\ge \frac{1}{\epsilon}\ln(\frac{1}{\delta})$, then the probability of the latter event is $(P(A^c))^m\le(1-\epsilon)^m\le
e^{-\epsilon m} \leq \delta$.
\end{proof}
\end{example}

This example shows that the gap between two different ERM learners can be as large as the gap between the Natarajan dimension and the graph dimension. By considering $\ch_\cx$ with an infinite $\cx$, we conclude the following corollary.
\begin{corollary}
There exist sets $\cx$, $\cy$ and a hypothesis class $\ch\subseteq \cy^\cx$, such that $\ch$ is learnable by some ERM learner but is not learnable by some other ERM learner.
\end{corollary}

In Example \ref{ex:baderm}, the bad ERM indeed requires as many examples as the graph dimension, while the good ERM requires only as many as the Natarajan dimension. Do such a `bad' ERM and a `good' ERM always exist? 
Our next result answers the question for the `bad' ERM in the affirmative. Indeed, the graph dimension determines the learnability of $\ch$ using the {\em worst} ERM learner.

\begin{theorem}\label{th:ERM-Bounds} There are constants $C_1,C_2>0$ such that the following holds. 
For every hypothesis class $\ch\subseteq \cy^\cx$ of Natarajan dimension $\ge 2$, 
 there exists an ERM learner $\ca_\bad$ such that
 for every $\epsilon<\frac{1}{12}$ and $\delta<\frac{1}{100}$, 
\[
C_1\left(\frac{d_G(\ch)+\ln(\frac{1}{\delta})}{\epsilon}\right)\le
m^r_{\ca_\bad}(\epsilon,\delta)\le m^r_{\erm}(\epsilon,\delta)\le C_2\left(\frac{d_G(\ch)\ln(\frac{1}{\epsilon})+\ln(\frac{1}{\delta})}{\epsilon}\right).
\]
\end{theorem}
\begin{proof} The upper bound is simply a restatement of Theorem \ref{th:multiclaas-simple-bounds}. It remains to prove that there exists an ERM learner, ${\ca}_{\bad}$, with $m^r_{\ca_{bad}}(\epsilon,\delta)\ge C_1\left(\frac{d_G(\ch)+\ln(\frac{1}{\delta})}{\epsilon}\right)$.

First, assume that $d=d_G(\ch)<\infty$.
Let $S=\{x_0,\ldots,x_{d-1}\}\subseteq \cx$ be a set which is $G$-Shattered by $\ch$ using the function $f_0$. Let $\ca_{\bad}$ be an ERM learner with the following property. Upon seeing a sample $T \subseteq S$ which is consistent with $f_0$, $\ca_{\bad}$ returns a function that coincides with $f_0$ on $T$ and disagrees with $f_0$ on $S \setminus T$. Such a function exists since $S$ is G-shattered using $f_0$.

Fix $\delta < \frac{1}{100}$ and $\epsilon<\frac{1}{12}$. Note that $1-2\epsilon
\ge e^{-4\epsilon}$.  Define a distribution on $\cx$ by setting
$\Pr(x_0)=1-2\epsilon$ and for all $1\le i\le d-1$,
$\Pr(x_i)=\frac{2\epsilon}{d-1}$.  Suppose that the correct hypothesis is $f_0$
and let $\{(X_i,f_0(X_i))\}_{i=1}^m$ be a sample.  Clearly, the hypothesis returned by $\ca_{\bad}$
will err on all the examples from $S$ which are not in the sample.
By Chernoff's bound, if $m\le \frac{d-1}{6\epsilon}$, then with probability at least $\frac{1}{100} \ge \delta$, the sample will include no more than
$\frac{d-1}{2}$ examples from $S\setminus\{x_0\}$, so that the returned hypothesis will have error
at least $\epsilon$. To see that, define r.v. $Y_i,\;1\le i\le m$ by setting $Y_i=1$ if $X_i\ne x_0$ and $0$ otherwise. By Chernoff's bound, if $r=\lfloor \frac{d-1}{6\epsilon} \rfloor$ then
$$\Pr\left(\sum_{i=1}^mY_i\ge\frac{d-1}{2} \right)\le \Pr\left(\sum_{i=1}^rY_i\ge 3\epsilon k \right)\le\exp\left(-\frac{\frac{1}{2}^2}{3}2\epsilon r\right)<0.99$$

Moreover, the probability that the sample includes only $x_0$
(and thus $\ca_{\bad}$ will return a hypothesis with error $2\epsilon$) is
$(1-2\epsilon)^m\ge e^{-4\epsilon m}$, which is more than $\delta$ if $m\le
\frac{1}{4\epsilon}\ln(\frac{1}{\delta})$. We therefore obtain that
\[
m^r_{{\ca}_{\bad}}(\epsilon,\delta)\ge \max\left\{\frac{d-1}{6\epsilon} , \frac{1}{2\epsilon}\ln(1/\delta)\right\} \ge
\frac{d-1}{12\epsilon} + \frac{1}{4\epsilon}\ln(1/\delta) ~,
\]
as required.

If $d_G(\ch)=\infty$, let $S_n,\;n=2,3,\ldots$ be a sequence of pairwise disjoint shattered sets such that $|S_n|=n$. For every $n$, suppose that $f_0^n$ indicated that $S_n$ is $G$-shattered. 
Let $\ca_{\bad}$ be an ERM learner with the following property. Upon seeing a sample $T \subseteq S_n$ labeled by $f^n_0$, $\ca_{\bad}$ returns a function that coincides with $f^n_0$ on $T$ and disagrees with $f_0$ on $S_n \setminus T$. Repeating the argument of the finite case for $S_n$ instead of $S$ shows that for every $\epsilon<\frac{1}{12}$ and $\delta<\frac{1}{100}$ it holds that $m_{\ca_\bad}(\epsilon,\delta)\ge C_1\left(\frac{n+\ln(\frac{1}{\delta})}{\epsilon}\right)$. Since it holds for every $n$, we conclude that $m^r_{\ca_\bad}(\epsilon,\delta)= \infty$.
\end{proof}
To get the sample complexity lower bound for the ERM learner $\ca_{\bad}$ in Example \ref{ex:baderm}, observe that this algorithm satisfies the specifications of a bad ERM algorithm from the proof above.

We conclude that for any multiclass learning problem there exists a `bad' ERM learner. 
The existence of `good' ERM learners turns out to be a more involved question. We conjecture that for every class there exists a `good' ERM learner -- that is, a learning algorithm whose realizable sample complexity is $\tilde{O}\left(\frac{d_N}{\epsilon}\right)$ (where the $\tilde{O}$ notation may hide poly-logarithmic factors of $\frac{1}{\epsilon},d_N$ and $1/\delta$ but {\em not} of $|Y|$).
As we describe in the rest of this section, in this work we prove this conjecture for several families of hypothesis classes.

What is the crucial feature that makes $\ca_{\good}$ better than $\ca_{\bad}$ in Example \ref{ex:baderm}? For the realizable case,
if the correct labeling is $f_A\in \ch_\cx$, then for {\em any} sample, $\ca_{\good}$ would return only one of at most two functions: either $f_A$ or $f_\emptyset$. On the other hand, if the correct labeling is $f_\emptyset$, then $\ca_{\bad}$ might return {\em every} function in $\ch_\cx$. Thus, to return a hypothesis with error at most $\epsilon$, $\ca_{\good}$ needs to reject at most one hypothesis, while $\ca_{\bad}$ might need to reject many more. Following this intuition, we propose the following rough principle: \emph{A good ERM learner is one that, for every target hypothesis, 
considers a small number of hypotheses.}

We would like to use this intuition to design ERMs with a better sample complexity than the one that can be guaranteed for a general ERM as in Theorem \ref{th:multiclass-bounds-our}. 
Classical sample complexity upper bounds that hold for all ERM learners hinge on the notion of a \emph{growth function}, which counts the number of different hypotheses induced by the hypothesis class on a sample of a certain size. 
To bound the sample complexity of a specific ERM learner, we define algorithm-dependent variants of the concept of a growth function.
\begin{definition}[Algorithm-dependent growth function]\label{def:algrowth}
Fix a hypothesis class $\ch\subseteq \cy^\cx$. Let $\ca$ be a learning algorithm for $\ch$. 
For $m>0$ and a sample $S = ((x_i,y_i))_{i = 1}^{2m}$ of size $2m$, let $\cx_S = \{x_1,\ldots,x_{2m}\}$, and define 
\[
F_\ca(S)=\{\ca(S')|_{\cx_S}\mid S'\subseteq S,\;|S'|=m\}.
\]
Let $R(\ch)$ be the set of samples which are consistent with $\ch$, that is $S = ((x_i,f(x_i)))_{i=1}^{2m}$ for some $f \in \ch$.
Define the \emph{realizable algorithm-dependent growth function} of $\ca$ by
\[
\Pi^r_\ca(m)=\sup_{ S \in R(\ch),|S|=2m}|F_\ca( S)|.
\]
Define the \emph{agnostic algorithm-dependent growth function} of $\ca$ for sample $S$ by
\[
\Pi^a_\ca(m)=\sup_{ S \in (\cx \times \cy)^{2m}}|F_\ca( S)|.
\]

\end{definition}
These definitions enable the use of a `double sampling' argument, similarly to the one used with the classical growth function \cite[see][chapter 4]{AB}. This argument is captured by the following lemma.
\begin{lemma} [The Double Sampling Lemma]\label{lem:doublesample}
Let $\ca$ be an ERM learner, and let $\cd$ be a distribution over $\cx\times\cy$. Denote $\epsilon = \Err_{\cd}(\ca(S_m))  - \Err_{\cd}(\ch)$, and let $\delta \in (0,1)$.
\begin{enumerate}
\item If $\cd$ is realizable by $\ch$ then with probability at least $1-\delta$, 
\[
\epsilon \leq 12\ln(2\Pi^r_\ca(m)/\delta)/m.
\]
\item 
For any $\cd$, with probability at least $1-\delta$, 
\[
\epsilon \leq \sqrt{\frac{32\ln((4\Pi^a_\ca(m)+4)/\delta)}{m}}.
\]
\end{enumerate}
\end{lemma}

\begin{proof}
The proof idea of the this lemma is similar to the one of the `double sampling' results of \citet{AB} (see their Theorems 4.3 and 4.8). 

For the first part of the claim, let $\cd$ be a realizable distribution for $\ch$. For $m \leq 8$, the claim trivially holds, therefore assume $m \geq 8$. 
Let $\nu = 12\ln(2\Pi_\ca^r(m)/\delta)/m$ and assume w.l.o.g.\ that $\nu \leq 1$.

Suppose that for some $S \in (\cx \times \cy)^m$, 
$\Err_\cd(\ca(S)) \geq \nu$. Let $T \in (\cx \times \cy)^m$ be another sample drawn from $D^m$, independently from $S$. We show that $Err_T(\ca(S)) \geq \nu/2$ with probability at least $\frac12$.
For $\nu \leq \frac12$, by Chernoff's bound, this holds with probability at least $1-\exp(-m\nu/16)$, which is larger than $\frac12$ by the definition of $\nu$. For $\nu \geq \frac12$, by Hoeffding's inequality, this holds with probability at least $1-\exp(-m\nu^2/2) \geq 1-\exp(-m/8)$, which is larger than $\frac12$, since $m \geq 8$.
It follows that
\begin{equation}\label{eq:st}
\frac12\Pr_{S \sim \cd^m}(\Err_\cd(\ca(S)) \geq \nu) \leq \Pr_{(S,T) \sim \cd^{2m}}(\Err_T(\ca(S)) \geq \nu/2).
\end{equation}

Let $Z =(z_1,\ldots,z_{2m})\in R(\ch)$, and let $\sigma:[2m]\rightarrow [2m]$ be a permutation. We write $Z_\sigma^1$ to mean $(z_{\sigma(1)},\ldots,z_{\sigma(m)})$ and $Z_\sigma^2$ to mean $(z_{\sigma(m+1)},\ldots,z_{\sigma(2m)})$. 

Similarly to Lemma 4.5 in \cite{AB}, for $\sigma$ drawn uniformly from the set of permutations,
\begin{align}\label{eq:perm1}
\Pr_{(S,T) \in \cd^{2m}}(\Err_T(\ca(S)) \geq \nu/2) &= \E_{Z \sim \cd^{2m}}(\Pr_{\sigma}(\Err_{Z_\sigma^2}(\ca(Z_\sigma^1)) \geq \nu/2)) \\
&\leq \sup_{Z \in R(\ch),|Z| = 2m}\Pr_{\sigma}(\Err_{Z_\sigma^2}(\ca(Z_\sigma^1)) \geq \nu/2).\notag
\end{align}

To bound the right hand side, note that since $\ca$ is an ERM algorithm, for any fixed $Z \in R(\ch)$ and any $\sigma$, $\Err_{Z_\sigma^1}(\ca(Z_\sigma^1)) = 0$.
Thus 
\[
\Pr_{\sigma}(\Err_{Z_\sigma^2}(\ca(Z_\sigma^1)) \geq \nu/2) \leq 
\Pr_{\sigma}(\exists h \in F_\ca(Z), \, \Err_{Z_\sigma^1}(h) =0 \text{ and }\Err_{Z_\sigma^2}(h) \geq \nu/2).
\]
For any fixed $h$, if the right hand side is not zero, 
then there exist at least $\nu m/2$ elements $(x,y)$ in $Z$ such that $h(x) \neq y$. In the latter case, the probability (over $\sigma$) that all such elements are in $Z_\sigma^2$ is at most $2^{-\nu m/2}$. With a union bound over $h \in F_\ca(Z)$, we conclude that for any $Z$, 
\[
\Pr_{\sigma}(\Err_{Z_\sigma^2}(\ca(Z_\sigma^1)) \geq \nu/2) \leq |F_\ca(Z)|2^{-\nu m/2}.
\]
Combining with \eqref{eq:perm1} gives
\[
\Pr_{(S,T) \in \cd^{2m}}(\Err_T(\ca(S)) \geq \nu/2) \leq \sup_{Z \in R(\ch)}|F_\ca(Z)|2^{-\nu m/2} = \Pi^r_\ca(m)2^{-\nu m/2}.
\]
By \eqref{eq:st} and the definition of $\nu$,
\[
\Pr_{S \sim \cd^m}(\Err_\cd(\ca(S)) \geq \nu) \leq 2\Pi^r_\ca(m)2^{-\nu m/2} \leq \delta.
\]
This proves the first part of the claim.

For the second part of the claim, let $\cd$ be a distribution over $\cx \times \cy$. Denote $\epsilon^* = \Err_\cd(\ch)$, and let $h^* \in \ch$ such that $\Err_\cd(h^*) = \epsilon^*$. 

Let $\nu = \sqrt{\frac{32\ln((4\Pi^a_\ca(m)+4)/\delta)}{m}}$. Suppose that for some $S \in (\cx \times \cy)^m$, 
$\Err_\cd(\ca(S)) \geq \epsilon^* + \nu$. Let $T \in (\cx \times \cy)^m$ be a random sample drawn from $D^m$ independently from $S$.
By Hoeffding's inequality, with probability at least $1-\exp(-m\nu^2/2)$, which is at least $\frac12$ by the definition of $\nu^2$, $\Err_T(\ca(S)) \geq \epsilon^*+\nu/2$.
It follows that
\begin{equation}\label{eq:st2}
\frac12\Pr_{S \sim \cd^m}(\Err_\cd(\ca(S)) \geq \epsilon^*+\nu) \leq \Pr_{(S,T) \sim \cd^{2m}}(\Err_T(\ca(S)) \geq \epsilon^*+\nu/2).
\end{equation}

Let $Z =(z_1,\ldots,z_{2m})\in (\cx \times \cy)^{2m}$, and let $\sigma:[2m]\rightarrow [2m]$ be a permutation. Denote $Z_\sigma^1$ and $Z_\sigma^2$ as above.

Denote $\cz = \{ Z \in (\cx \times \cy)^{2m} \mid \Err_{Z}(\ca(Z_\sigma^1)) \leq \epsilon^* + \nu/8\}$.
By lemma 4.5 in \cite{AB} again, 
for $\sigma$ drawn uniformly from the set of permutations,
\begin{align}\label{eq:perm3}
&\Pr_{(S,T) \in \cd^{2m}}(\Err_T(\ca(S)) \geq \epsilon^*+\nu/2) = \E_{Z \sim \cd^{2m}}(\Pr_{\sigma}(\Err_{Z_\sigma^2}(\ca(Z_\sigma^1)) \geq \epsilon^*+\nu/2))\\
&\qquad\qquad\leq \E_{Z \sim \cd^{2m}}\left(\Pr_{\sigma}(\Err_{Z_\sigma^2}(\ca(Z_\sigma^1)) \geq \epsilon^*+\nu/2) \Big| Z \in \cz\right) + \Pr(Z \notin \cz).\notag
\end{align}
To bound the right hand side, first note that by Hoeffding's inequality, the second term is bounded by
\begin{equation}\label{eq:notinz}
 \Pr(Z \notin \cz) \leq \exp(-\nu^2m/16).
\end{equation}
For the first term, $\Err_{Z_\sigma^2}(\ca(Z_\sigma^1)) \geq \epsilon^*+\nu/2$ implies that unless $\Err_{Z_\sigma^1}(\ca(Z_\sigma^1)) > \epsilon^* + \nu/4$, necessarily
$\Err_{Z_\sigma^2}(\ca(Z_\sigma^1)) - \Err_{Z_\sigma^1}(\ca(Z_\sigma^1)) \geq\nu/4$.
Since $\ca$ is an ERM algorithm, $\Err_{Z_\sigma^1}(\ca(Z_\sigma^1)) > \epsilon^* + \nu/4$ only if also $\Err_{Z_\sigma^1}(h^*) > \epsilon^* + \nu/4$.
Therefore, for any $Z$,
\begin{align}\label{eq:sumeqs}
&\Pr_{\sigma}(\Err_{Z_\sigma^2}(\ca(Z_\sigma^1)) \geq \epsilon^*+\nu/2) \leq \notag\\
&\quad\Pr_\sigma(\Err_{Z_\sigma^1}(h^*) > \epsilon^* + \nu/4) + \Pr_\sigma(\Err_{Z_\sigma^2}(\ca(Z_\sigma^1)) - \Err_{Z_\sigma^1}(\ca(Z_\sigma^1)) > \nu/4).
\end{align}
$\Err_{Z_\sigma^1}(h^*)$ is an average of $m$ random variables of the form $\onefunc[h^*(x_i) \neq y_i]$, that are sampled without replacement from the finite population $Z$, with population average 
 $\Err_Z(h^*)$. For $Z \in \cz$, $\Err_Z(h^*) \leq \epsilon^*+\nu/8$.
Therefore, by Hoeffding's inequality for sampling without replacements from a finite population \citep{Hoeffding63}, for $Z \in \cz$, 
\begin{equation}\label{eq:hstar}
\Pr_\sigma(\Err_{Z_\sigma^1}(h^*) > \epsilon^* + \nu/4) \leq \Pr_\sigma(\Err_{Z_\sigma^1}(h^*) - \Err_Z(h^*) > \nu/8) \leq \exp(-\nu^2m/32).
\end{equation}
In addition, by the same inequality, and applying the union bound over $h \in F_\ca(Z)$, for any $Z$
\begin{align}
& \Pr_\sigma(\Err_{Z_\sigma^2}(\ca(Z_\sigma^1)) - \Err_{Z_\sigma^1}(\ca(Z_\sigma^1)) > \nu/4) \leq \Pr_\sigma(\exists h\in F_\ca(Z), \Err_{Z_\sigma^2}(h) - \Err_{Z_\sigma^1}(h) > \nu/4) \notag\\
&\quad\leq 
\Pr_\sigma(\exists h\in F_\ca(Z), \Err_{Z_\sigma^2}(h) - \Err_Z(h) > \nu/8) +\Pr_\sigma(\exists h\in F_\ca(Z), \Err_{Z_\sigma^1}(h) - \Err_Z(h) > \nu/8)\notag\\
&\quad\leq 2\Pi^a_\ca(m) \exp(- \nu^2 m/32).\label{eq:errdiff}
\end{align}

Combined with \eqref{eq:sumeqs} and \eqref{eq:hstar},
it follows that for $Z \in \cz$,
\[
\Pr_{\sigma}(\Err_{Z_\sigma^2}(\ca(Z_\sigma^1)) \geq \epsilon^*+\nu/2) \leq 
(2\Pi^a_\ca(m)+1) \exp(- \nu^2 m/32).
\]
With \eqref{eq:st2}, \eqref{eq:perm3}, and \eqref{eq:notinz}, we conclude that 
\[
\Pr_{S \sim \cd^m}(\Err_\cd(\ca(S)) \geq \epsilon^*+\nu) 
\leq (4\Pi^a_\ca(m)+4)  \exp(-\nu^2 m/32) \equiv \delta.
\]
The claim follows since $\epsilon = \Err_\cd(\ca(S)) - \epsilon^*$.
\end{proof}

As we shall presently see, \lemref{lem:doublesample} can be used to provide better sample complexity bounds for some `good' ERM learners. 

\subsection{Learning with a small essential range}
A key tool that we will use for providing better bounds is the notion of {\em essential range}, defined below.
The essential range of an algorithm quantifies the number of different labels that can be emitted by the functions the algorithm might return for samples of a given size.
In this definition we use the notion of the range of a function. Formally, for a function $f:\cx \rightarrow \cy$, its range is the set of labels to which it maps $\cx$, denoted by $\range(f) = \{ f(x) \mid x \in \cx\}$. 
\begin{definition}[Essential range]
Let $\cal A$ be a learning algorithm for $\ch\subseteq \cy^{\cx}$. The {\em realizable essential range} of $\cal A$ is the function $r^r_\ca:\mathbb N\to\mathbb N$, defined as follows.
\[
r^r_\ca(m) = \sup_{S \in R(\ch), |S| = 2m} \left|\cup_{S'\subset S,\;|S'|=m}\operatorname{range}(\ca(S'))\right|.
\]
The {\em agnostic essential range} 
of $\cal A$ is the function $r^a_\ca:\mathbb N\to\mathbb N$, defined as follows.
\[
r^a_\ca(m) = \sup_{S \subseteq \cx \times \cy, |S| = 2m} \left|\cup_{S'\subset S,\;|S'|=m}\operatorname{range}(\ca(S'))\right|.
\]
\end{definition}
Intuitively, an algorithm with a small essential range uses a smaller set of labels for any particular distribution, thus it enjoys better convergence guarantees. This is formally quantified in the following result.

\begin{theorem} \label{thm:restricted_range} Let $\ca$ be an ERM learning algorithm for $\ch \subseteq \cy^{\cx}$ with essential ranges $r^r_\ca(m)$ and $r^a_\ca(m)$. Denote $\epsilon = \Err_{\cd}(\ca(S_m))  - \Err_{\cd}(\ch)$. 
Then,
\begin{itemize}
\item If $\cd$ is realizable by $\ch$ and $\delta < 0.1$ then with probability at least $1-\delta$, 
\[
\epsilon \leq O\left(\frac{d_N(\ch)(\ln(m) + \ln(r^r_\ca(m)))+\ln(1/\delta)}{m}\right).
\]
\item 
For any probability distribution $D$, with probability at least $1-\delta$, 
\[
\epsilon \leq O\left(\sqrt{\frac{d_N(\ch)(\ln(m) + \ln(r^a_\ca(m))+ \ln(1/\delta)}{m}}\right).
\]
\end{itemize}
\end{theorem}

To prove the realizable part of this theorem, we use the following combinatorial lemma by Natarajan:
\begin{lemma}(\citealp{Natarajan89b})\label{lemma:growth-function}
For every hypothesis class $\ch\subseteq \cy^\cx$, $|\ch|\le |\cx|^{d_N(\ch)}|\cy|^{2d_N(\ch)}$.
\end{lemma}

\begin{proof}[of \thmref{thm:restricted_range}]
For the realizable sample complexity, the growth function can be bounded as follows.
Let $S \in R(\ch)$ such that $|S| = 2m$, and consider the function class $F_\ca(S)$ (see \defref{def:algrowth}). By definition, the domain of $F_\ca(S)$ is $\cx_S$ of size $2m$, and the range of $F_\ca(S)$ is of size at most $r^r_\ca(m)$. Lastly, the Natarajan dimension of $F_\ca(S)$ is at most $d_N(\ch)$, since $F_\ca(S) \subseteq \ch|_S$.

Therefore, by \lemref{lemma:growth-function}, $|F_\ca(S)| \leq (2m)^{d_N(\ch)}r^r_\ca(m)^{2d_N(\ch)}$.
Taking the supremum over all such $S$, we get
\[
\Pi^r_\ca(m) \leq (2m)^{d_N(\ch)} r^r_\ca(m)^{2d_N(\ch)}.
\]
The bound on $\epsilon$ follows from the first part of \lemref{lem:doublesample}.

For the agnostic sample complexity, a similar argument shows that
\[
\Pi^a_\ca(m) \leq (2m)^{d_N(\ch)} r^a_\ca(m)^{2d_N(\ch)},
\]
and the bound on $\epsilon$ follows from the second part of \lemref{lem:doublesample}.
\end{proof}

Theorem \ref{th:multiclass-bounds-our}, which provides an improved bound for the realizable case, now follows from the fact that the essential range is never more than $k$. But the essential range can also be much smaller than $k$. For example, the essential range of the algorithm from Example \ref{ex:baderm} is bounded by $2m+1$ (the $2m$ labels appearing in the sample together with the $*$ label). In fact, we can state a more general bound, for any algorithm which never `invents' labels it did not observe in the sample.
\begin{corollary}\label{cor:only_observed_labels}
Let $\ca$ be an ERM learner for a hypothesis class $\ch\subseteq\cy^\cx$. Suppose that for every sample $S$, the function $\ca(S)$ never outputs labels which have not appeared in $S$. Then
\[
m^r_\ca(\epsilon,\delta)= O\left(\frac{d_N(\ch)(\ln(\frac{1}{\epsilon})+\ln(d_N(\mathcal H)))+\ln(\frac{1}{\delta})}{\epsilon}\right),
\]
and 
\[
m^a_\ca(\epsilon,\delta)= O\left(\frac{d_N(\ch)(\ln(\frac{1}{\epsilon})+\ln(d_N(\mathcal H)))+\ln(\frac{1}{\delta})}{\epsilon^2}\right).
\]
\end{corollary}
This corollary is immediate from \thmref{thm:restricted_range} by setting $r^r_\ca(m) = r^a_\ca(m) = 2m$.

From this corollary, we immediately get that every hypothesis class which admits such algorithms, and has a large gap between the Natarajan dimension and the graph dimension realizes a gap between the sample complexities of different ERM learners. Indeed, the graph dimension can even be unbounded, while the Natarajan dimension is finite and the problem is learnable. This is demonstrated by the following example.

\begin{example} \label{ex:baderm2} 
Denote the ball in $\reals^n$ with center $z$ and radius $r$ by $B_n(z,r) = \{ x \mid \|x - z\| \leq r\}$.
For a given ball $B=B_n(z,r)$ with $z \in \reals^n$ and $r > 0$, let $h_B:\reals^n \rightarrow \reals^n \cup \{*\}$ be the function defined by $h_B(x)=z$ if $x \in B$ and $h_B(x)= *$ otherwise. Let $h_*$ be a hypothesis that always returns $*$. Define the hypothesis class $\ch_{n}$ of hypotheses from $\reals^n$ to $\reals^n \cup \{*\}$ by 
\[
  \ch_n = \{h_B \mid \exists z \in \reals^n, \infty\ge r > 0,\text{ such that } B=B_n(z,r)\} \cup \{h_{*} \}.
\] 
Relying on the fact that the VC dimension of balls in $\reals^n$ is $n+1$, it is not hard to see that $d_G(\ch_n)=n+1$. Also, it is easy to see that $d_N(\ch_n)=1$.
It is not hard to see that there exists an ERM, $\ca_{\good}$, satisfying the requirements of Corollary \ref{cor:only_observed_labels}. Thus,
\[
m^r_{\ca_{\good}}(\epsilon, \delta) \leq O \left(\frac{\ln(1/\epsilon) +\ln(1/\delta)}{\epsilon}\right),\;\;m^a_{\ca_{\good}}(\epsilon, \delta) \leq O \left(\frac{\ln(1/\delta)}{\epsilon^2}\right).
\]
On the other hand, Theorem \ref{th:ERM-Bounds} implies that there exists a bad ERM learner, $\ca_{\bad}$ with
\[
m^a_{\ca_{\bad}}(\epsilon, \delta) \geq m^r_{\ca_{\bad}}(\epsilon, \delta) \geq C_1 \left(\frac{n +\ln(1/\delta)}{\epsilon}\right).
\]
\end{example}

Our results so far show that whenever an ERM learner with a small essential range exists, the sample complexity of learning the multiclass problem can be improved over the worst ERM learner. In the next section we show that this is indeed the case for hypothesis classes which satisfy a natural condition of \emph{symmetry}.

\subsection{Learning with Symmetric Classes}\label{sec:symmetric}
We say that a hypothesis class $\ch$ is symmetric if for any $f\in \ch$ and any permutation $\phi:\cy \rightarrow \cy$ on labels\, we have that $\phi\circ f\in\ch$ as well. Symmetric classes are a natural choice if there is no prior knowledge on properties of specific labels in $\cy$ (See also the discussion in \secref{sec:symmetrization} below).
We now show that for symmetric classes, the Natarajan dimension characterizes the optimal  sample complexity up to logarithmic factors. It follows that a finite Natarajan dimension is a necessary and sufficient condition for learnability of a symmetric class. 
We will make use of the following lemma, which provides a key observation on symmetric classes.
\begin{lemma}\label{lemma:attain-at-most-d-values}
Let $\ch\subseteq \cy^\cx$ be a symmetric hypothesis class of Natarajan dimension $d$. Then any $h \in \ch$ has a range of size at most $2d+1$.
\end{lemma}
\begin{proof}
If $k\le 2d+1$ we are done. Thus assume that there are $2d+2$ distinct
elements $y_1,\ldots,y_{2d+2}\in \cy$. Assume to the contrary that there is a
hypothesis $h\in\ch$ with a range of more than $2d+1$ values. Thus there is a set
$S=\{x_1,\ldots,x_{d+1}\}\subseteq \cx$ such that $h|_S$ has $d+1$ values in
its range. Since $\ch$ is symmetric, we can show that $\ch$ N-shatters $S$ as follows: 
Since $\ch$ is symmetric, we can rename all the labels in the range of $h|_S$ as we please and get another function in $\ch$. Thus there are two 
functions $f_1,f_2\in \ch$ such that for all $i \leq d+1$, $f_1(x_i)=y_i$ and $f_2(x_i) = y_{d+1+i}$. 
Now, let $S \subseteq T$. Since $\ch$ is symmetric we can again rename the labels in the range of $h|_S$ 
to get a function $g \in \ch$ such that $g(x)=f_1(x)$ for every $x\in
T$ and $g(x)=f_2(x)$ for every $x\in S\setminus T$. Therefore the set $S$ is shattered, thus
the Natarajan dimension of $\ch$ is at least $d+1$, contradicting the assumption.
\end{proof}
First, we provide an upper bound on the sample complexity of ERM in the realizable case.
\begin{theorem}\label{th:symmetric}
There are absolute constants $C_1,C_2$ such that for every symmetric hypothesis class $\ch\subseteq \cy^\cx$
\[
C_1\left(\frac{d_N(\ch)+\ln(\frac{1}{\delta})}{\epsilon}\right)\le
m^r_{\erm}(\epsilon,\delta)\le C_2\left(\frac{d_N(\ch)\left(\ln(\frac{1}{\epsilon})+\ln(d_N(\ch))\right)+\ln(\frac{1}{\delta})}{\epsilon}\right)\]
\end{theorem}

\begin{proof}
The lower bound is a restatement of \thmref{th:multiclaas-simple-bounds}.  For the
upper bound, first note that if $k \leq 4d_N(\ch)+2$ the upper bound trivially follows from \thmref{th:multiclass-bounds-our}. Thus assume $k > 4d_N(\ch) +2$.
We define an ERM learner $\ca$ with a small essential range, as required in 
\thmref{thm:restricted_range}: Fix a set $Z \subseteq \cy$ of size $|Z|=2d_N(\ch)+1$. Assume an input sample $(x_1,f(x_1)),\ldots,(x_m,f(x_m))$, and denote the set of labels that appear in the sample by $L = \{f(x_i) \mid i \in [m]\}$. We require that $\ca$ return a hypothesis which is consistent
with the sample and has range in $L\cup Z$. 

To see that such an ERM learner exists, observe that by \lemref{lemma:attain-at-most-d-values}, the range of $f$ has at most $2d_N(\ch) +1$ distinct labels. 
Therefore, there is a set $R \subseteq \cy$ such that $|R| \leq 2d_N(\ch)+1$ and the range of $f$ is $L \cup R$. Due to the symmetry of $\ch$, we can rename the labels in $R$ to labels in $Z$, and get another function $g \in \ch$, that is consistent with the sample and has range in $L \cup Z$. This function can be returned by $\ca$.

The range of $\ca$ over all samples that are labeled by a fixed function $f \in \ch$ is thus in the union of $Z$ and the range of $f$. $|Z| \leq 2d_N(\ch) + 1$ and by Lemma \ref{lemma:attain-at-most-d-values}, the range of $f$ is also at most $2d_N(\ch) + 1$. Therefore the realizable essential range of $\ca$ is at most $4d_N(\ch)+2$. The desired bound for the sample complexity of $\ca$ thus follows from Theorem \ref{thm:restricted_range}.

We now show that the same bound in fact holds for all ERM learners for $\ch$. 
Suppose that $\ca'$ is an ERM learner for which the bound does not hold. Then there is a function $f$ and a distribution $D$ over $\cx \times \cy$ which is consistent with $f$, and there are $m,\epsilon$ and $\delta$ for which $m \geq m_\ca^r(\epsilon,\delta)$, such that with probability greater than $\delta$ over samples $S_m$, 
$\Err_{\cd}(\ca'(S_m))  - \Err_{\cd}(\ch) > \epsilon$.
Consider $\ca$ as defined above, with a set $Z$ that does not overlap with the range of $f$. For every sample $S_m$ consistent with $f$, denote $\hat{f} = \ca'(S_m)$, and let $\ca$ return $g$ which results from renaming the labels in $\hat{f}$ as follows: For any label that appeared in $S_m$, the same label is used in $g$. For any label that did not appear in $S_m$, a label from $Z$ is used instead. Clearly, $\Err_{\cd}(\ca(S_m)) \geq \Err_{\cd}(\ca'(S_m))$. But this contradicts the upper bounds on $m_\ca^r(\epsilon,\delta)$. We conclude that the upper bound holds for all ERM learners. 
\end{proof}
Second, we have the following upper bound for the agnostic case.
\begin{theorem}\label{thm:symmetricag}
There are absolute constants $C_1,C_2$ such that for every symmetric hypothesis class $\ch\subseteq \cy^\cx$
\[
C_1\left(\frac{d_N(\ch)+\ln(\frac{1}{\delta})}{\epsilon^2}\right)\le
m_\erm^a(\epsilon,\delta)\leq C_2\left(\frac{d_N(\ch)\ln(\min\{d_N(\ch),k\}) +\ln(\frac{1}{\delta})}{\epsilon^2}\right),
\]

\end{theorem}
\begin{proof}\footnote{We note that this proof show that for symmetric classes $d_G=O\left(d_N\log(d_N)\right)$. Hence, it can be adopted to give a simpler proof of theorem \ref{th:symmetric}, but with a multiplicative (rather than additive) factor of $\log\left(\frac{1}{\epsilon}\right)$.}
The lower bound is a restatements of Theorem \ref{th:multiclaas-bounds}.  For the
upper bound, first note that if $k \leq 6d_N(\ch)$ then the upper bound 
follows from \thmref{th:multiclaas-bounds}. Thus assume  $k \geq 6d_N(\ch) \geq 4d_N(\ch)+2$.
Fix a set $Z \subseteq \cy$ of size $|Z| = 4d_N(\ch)+2$. Denote $\ch'=\{f\in\ch:f(\cx)\subseteq Z\}$. By Lemma \ref{lemma:attain-at-most-d-values}, the range of every function in $\ch$ contains at most $\frac{|Z|}{2}$ elements. Thus, by symmetry, it is easy to see that $d_G(\ch)=d_G(\ch ')$ and $d_N(\ch)=d_N(\ch')$. 
By equation (\ref{eq:ben-david}) and the fact that the range of functions in $\ch'$ is $Z$, we conclude that 
\begin{align*}
d_G(\ch)&=d_G(\ch') = O(d_N(\ch')\ln(|Z|))  \\
&= O(d_N(\ch')\ln(\min\{d_N(\ch'),k\}) =  O(d_N(\ch)\ln(d_N(\ch))).
\end{align*}
Using \thmref{th:multiclaas-simple-bounds} we obtain the desired upper bounds.
\end{proof}
These results indicate that for symmetric classes, the sample complexity is determined by the Natarajan dimension up to logarithmic factors. Moreover, the ratio between the sample complexities of worst ERM and the best ERM in this case is also at most logarithmic in $\epsilon$ and the Natarajan dimension. 
We present the following open question:
\begin{open question}
Are there symmetric classes such that there are two different ERM learners with a sample complexity ratio of $\Omega(\ln(d_N))$ between them?
\end{open question}

\subsection{Learning with No Prior Knowledge on Labels}\label{sec:symmetrize}

Suppose we wish to learn some multiclass problem and have some hypothesis class that we wish to use for learning. The hypothesis class is defined using arbitrary label names, say $\cy = \{1,\ldots,k\} = [k]$. In many learning problems, we do not have any prior knowledge on a preferred mapping between these arbitrary label names and the actual real-world labels (e.g., names of topics of documents). Thus, any mapping between the real-world class labels and the arbitrary labels in $[k]$ is as reasonable as any other. We formalize the last assertion by {\em assuming that this mapping is chosen uniformly at random
}\footnote{
We note also that choosing this mapping at random is sometimes advocated for multiclass learning, e.g., for a filter tree \cite{BeygelzimerLaRa07} and for an Error Correcting Output Code \citep{DietterichBa95,AllweinScSi00a}.}. 
In this section we show that in this scenario, when $k = \Omega(d_N(\ch))$, it is likely that we will achieve poor classification accuracy.

Formally, let $\ch\subset[k]^{\cx}$ be a hypothesis class. Let $\cl$ be the set of real-world labels, $|\cl| = k$. 
A mapping of the label names $[k]$ to the true labels $\cl$ is a bijection $\phi:[k]\to \cl$.
For such $\phi$ we let $\ch_\phi=\{\phi\circ f:f\in\ch\}$. \footnote{Several notions, originally defined w.r.t. functions from $\cx$ to $\cy$ (e.g. $\Err_{\cd}(h)$), can be naturally extended to functions from $\cx$ to $\cl$. We will freely use these extensions.}

The following theorem upper-bounds the approximation error when $\phi$ is chosen at random. The result holds for any distribution with fairly balanced label frequencies. Formally, we say that $\cd$ over $\cx\times \cl$ is \emph{balanced} if for any $l \in \cl$, the probability that a random pair drawn from $\cd$ has label $l$ is at most $10/k$.
\begin{theorem}\label{thm:split_classes}
Fix $\alpha>0$. There exist a constant $C_\alpha>0$ such that for any $k > 0$, any hypothesis class $\ch \subseteq [k]^\cx$ such that $d_N(\ch) \leq C_\alpha k$, and any balanced distribution $\cd$ over $\cx \times \cl$, with probability at least $1-o(2^{-k})$ over the choice of $\phi$,
$\Err_{\cd}(\ch_\phi)\ge 1-\alpha.$
\end{theorem}

\begin{remark}
Theorem \ref{thm:split_classes} is tight, in the sense that a similar proposition cannot be obtained for all $d_N \leq f(k)$ for some $f(k) \in \omega(k)$. To see this, consider the class $\ch=[k]^{[k]}$, for which $d_N(\ch) = k$. For any $\phi$, $\ch_\phi = \ch$. Thus, for any distribution such that  $\Err_\cd(\ch)=0$, we have $\Err_{\cd}(\ch_\phi)=0$. 
\end{remark}

To prove Theorem \ref{thm:split_classes}, we prove the following lemma, which provides a lower bound on the error of any hypothesis with a random bijection.

\begin{lemma}\label{lem:chern_substitute}
Let $h:\cx\to[k]$ and let $\phi:[k]\to\cl$ be a bijection chosen uniformly at random. Let $S=\{(x_1,l_1),\ldots,(x_m,l_m)\}\subseteq \cx\times\cl$. Denote, for $l\in \cl$, 
$\hat p_l=\frac{|\{j:l_j=l\}|}{m}$. Fix $\alpha > 0$, and let $\gamma=\frac{\alpha^2}{\sum_{l\in \cl}\hat{p}_l^2}$. Then
\[
\Pr[\Err_{S}(\phi\circ h)<1-\alpha]\le\left(\frac{8ke}{\gamma^2}\right)^{\frac{\gamma}{2}}.
\]
\end{lemma}
\begin{proof}
Denote $P = \sqrt{\sum_{l\in \cl}\hat{p}_l^2}$. 
For a sample $S\subset \cx\times\cl$ and a function $f:\cx\to\cl$ denote $\Gain_S(f)=1-\Err_S(f)$. For $l\in \cl$ denote $S_l=((x_i,l_i))_{i:l_i=l}$. 
By Cauchy-Schwartz, we have
\[
\Gain_S(\phi\circ h)=\sum_{l\in\cl}\hat p_l\cdot \Gain_{S_l}(\phi\circ h)\le P\cdot\sqrt{\sum_{l\in\cl}\left(\Gain_{S_l}(\phi\circ h)\right)^2}~.
\]
Assume that $\Err_S(\phi\circ h)\le 1-\alpha$. Then
\[
\sum_{l\in\cl}\Gain_{S_l}(\phi\circ h)\ge\sum_{l\in\cl}\left(\Gain_{S_l}(\phi\circ h)\right)^2
\ge \frac{\left(\Gain_S(\phi\circ h)\right)^2}{P^2}\ge \frac{\alpha^2}{P^2}=\gamma.
\]
Note first that the left hand side is at most $k$, thus $\gamma \leq k$.
Since for every $l\in\cl$ it holds that $0\le\Gain_{S_l}(\phi\circ h)\le 1$, we conclude that there are at least $n=\lceil\frac{\gamma}{2}\rceil$ labels $l\in \cl$ such that 
$$\Gain_{S_l}(\phi\circ h)\ge \frac{\gamma}{2k}~.$$
For a fixed set of $n$ labels $l_1,\ldots,l_n\in\cl$, the probability that $\forall i,\;\Gain_{S_{l_i}}(\phi\circ h)\ge \frac{\gamma}{2k}$ is at most
$$\prod_{i=1}^n\frac{2k}{(k+1-i)\gamma}\le\left(\frac{2k}{(k+1-n)\gamma}\right)^n~.$$
To see that, suppose that $\phi$ is sampled by first choosing the value of $\phi^{-1}(l_1)$ then $\phi^{-1}(l_2)$ and so on. For every $l_i$, there are at most $\frac{2k}{\gamma}$ values for $\phi^{-1}(l_i)$ for which $\Gain_{S_{l_i}}(\phi\circ h)\ge \frac{\gamma}{2k}$. Thus, after the values of $\phi^{-1}(l_1),\ldots,\phi^{-1}(l_{i-1})$ have been determined, the probability that $\phi^{-1}(l_i)$ is one of these values is at most $\frac{2k}{(k+1-i)\cdot\gamma}$.

It follows that the probability that $\Gain_{S_l}(\phi\circ h)\ge \frac{\gamma}{2k}$ for $n$ different labels $l$ is at most
\begin{eqnarray*}
\binom kn \cdot\left(\frac{2k}{(k+1-n)\gamma}\right)^n &\le &
\left(\frac{ek}{n}\right)^n\cdot \left(\frac{2k}{(k+1-n)\gamma}\right)^n\\
&\le &
\left(\frac{2ke}{\gamma}\right)^n\cdot \left(\frac{2k}{(k-\gamma/2)\gamma}\right)^n\\
&\le &
\left(\frac{8ke}{\gamma^2}\right)^{n}.
\end{eqnarray*}
If $\frac{8ke}{\gamma^2} \geq 1$ then the bound in the statement of the lemma holds trivially. Otherwise, the bound follows since $n \geq \gamma/2$.
\end{proof}

\begin{proof}[Proof of \thmref{thm:split_classes}]
Denote $p_{l}=\Pr_{(X,L)\sim\cd}[L=l]$. Let $S=\{(x_1,l_1),\ldots,(x_m,l_m)\}\subseteq\cx\times \cl$ be an i.i.d. sample drawn according to $\cd$. Denote $\hat p_l=\frac{|\{j:l_j=l\}|}{m}$.

For any fixed bijection $\phi$, by theorem \ref{th:multiclaas-bounds}, with probability $1-\delta$ over the choice of $S$, 
\[
\Err_\cd(\ch_\phi) \geq \inf_{h\in\ch}\operatorname{Err}_{S}(\phi\circ h) - O\left(\sqrt{\frac{\ln(k)d_N(\ch) + \ln(1/\delta)}{m}}\right).
\]
Since there are less than $k^k$ such bijections, we can apply the union bound to get that
with probability $1-\delta$ over the choice of $S$,
\[
\forall \phi,\quad\Err_\cd(\ch_\phi) \geq \inf_{h\in\ch}\operatorname{Err}_{S}(\phi\circ h) - O\left(\sqrt{\frac{\ln(k)d_N(\ch) + k\ln(k) + \ln(1/\delta)}{m}}\right).
\]
Assume $k \geq C\cdot d_N(\ch)$ for some constant $C > 0$, and let $m= \Theta\left(\frac{k\cdot\ln(k)}{\alpha^2}\right)$
such that with probability at least $3/4$,
\begin{equation}\label{eq:1}
\forall \phi,\quad\Err_\cd(\ch_\phi) \geq \inf_{h\in\ch}\operatorname{Err}_{S}(\phi\circ h) - \alpha/2.
\end{equation}
We have
\[
E[\sum_{l\in\cl}\hat{p}_l^2]=2\frac{1}{m^2}\sum_{l\in\cl}\left(\binom m2p_l^2+mp_l\right)\le 2k\cdot\left(\frac{m(m-1)}{2m^2}\frac{100}{k^2}+\frac{10}{mk}\right)\le \frac{120}{k}.
\]
Thus, by Markov's inequality, with probability at least $\frac{1}{2}$ over the samples we have \begin{equation}\label{eq:2}
\sum_{l\in\cl}\hat{p}_l^2<\frac{240}{k}.
\end{equation}
Thus, with probability at least $1/4$, both (\ref{eq:2}) and (\ref{eq:1}) hold. In particular, there exists a single sample $S$ for which both (\ref{eq:2})  and (\ref{eq:1}) hold. Let us fix such an $S=\{(x_1,l_1),\ldots,(x_m,l_m)\}$.

Assume now that $\phi:\cy\to\cl$ is sampled uniformly. For a fixed $h\in\ch$ and for $\gamma=(\alpha/2)^2/\sum_{l\in \cl} \hat{p}_l^2 \geq k\alpha^2/960$, we have, by Lemma \ref{lem:chern_substitute} that
\[
\Pr_{\phi}\left[\operatorname{Err}_{S}( \phi\circ h)<1-\frac{\alpha}{2}\right]\le \left(\frac{8ke}{\gamma^2}\right)^{\frac{\gamma}{2}} \leq (C_1k\alpha^4)^{-C_2 k\alpha^2} := \eta,
\]
for constants $C_1,C_2 > 0$.
By Lemma \ref{lemma:growth-function}, $|\ch|_{\{x_1,\ldots,x_m\}}|\le \left(m\cdot k\right)^{2d_N(\ch)}$. Thus, with probability $\ge 1-\left(m\cdot k\right)^{2d}\cdot\eta$ over the choice of $\phi$, $\inf_{h\in\ch}\operatorname{Err}_{S}(\phi\circ h)\ge 1-\frac{\alpha}{2}$ and by (\ref{eq:1}) also
\begin{equation}\label{eq:3}
\operatorname{Err}_{\cd}( \ch_\phi)\ge 1-\alpha.
\end{equation}
By our choice of $m$, and since $k \geq d_N(\ch)$, for some universal constant $C_1\ge 1$, $m\le C_1\cdot \frac{k^2}{\alpha^2}$. Considering $\alpha$ a constant, we have, for some constants $C_i > 0$,
\begin{align*}
\left(m\cdot k\right)^{2d_N(\ch)}\cdot\eta \leq (C_3 k)^{6d_N(\ch)} \cdot (C_4 k)^{-C_5 k}.
\end{align*}
By requiring that $k \geq 12d_N(\ch)/C_5$, we get that the right hand side is at most $o(2^{-k})$.

\end{proof}

\subsubsection{Symmetrization}\label{sec:symmetrization}
From Theorem \ref{thm:split_classes} it follows that if there is no prior knowledge about the labels, and the label frequencies are balanced, we must use a class of Natarajan dimension $\Omega(k)$ to obtain reasonable approximation error. As we show next, in this case, there is almost no loss in the sample complexity if one instead uses the {\em symmetrization} of the class, obtained by considering all the possible label mappings $\phi: [k] \rightarrow \cl$. Formally, let $\ch\subset [k]^\cx$ be some hypothesis class and let $\cl$ be a set with $|\cl|=k$. The symmetrization of $\ch$ is the symmetric class
\[
\ch_{\mathrm{sym}}=\{\phi\circ h\mid h\in\ch,\;\phi:[k]\to\cl\text{ is a bijection}\}.
\]
\begin{lemma}
Let $\ch \subseteq [k]^\cx$ be a hypothesis class with Natarajan dimension $d$. Then 
\[
d_N(\ch_{\mathrm{sym}})=O(\max\{d\log(d),k\log(k)\}).
\]
\end{lemma}
\begin{proof}
Let $d_s=d_N(\ch_{\mathrm{sym}})$. Let $X\subset \cx$ be a set of cardinality $d_s$ that is N-shattered by $\ch_{\mathrm{sym}}$. By Lemma \ref{lemma:growth-function}, $|\ch|_X|\le (d_sk^2)^d$. It follows that $|\ch_{\mathrm{sym}}|_X|\le k!(d_sk^2)^d$. On the other hand, since $\ch_{\mathrm{sym}}$ N-shatters $X$, $|\ch_{\mathrm{sym}}|_X|\ge 2^{|X|}=2^{d_s}$. It follows that $2^{d_s}\le k!(d_sk^2)^d$.
Taking logarithms we obtain that $d_s\le k\log(k)+d(\ln(d_s)+2\ln(k))$. The Lemma follows.
\end{proof}

\section{Other learning settings}\label{sec:other}

In this section we consider the characterization of learnability in other learning settings: The online setting and the bandit setting.

\subsection{The Online Model}\label{ch:online}
Learning in the online model is conducted in a sequence of consecutive
rounds. On each round $t=1,2,\ldots,T$, the environment presents a
sample $x_t\in \cx$, then the algorithm should predict a value $\hat{y_t}
\in \cy$, and finally the environment reveals the correct value $y_t \in
\cy$. The prediction at time $t$ can be based only on the examples
$x_1,\ldots,x_t$ and the previous outcomes $y_1,\ldots,y_{t-1}$. Our
goal is to minimize the number of prediction mistakes in the worst
case, where the number of mistakes on the first $T$ rounds is $L_T= |\{t \in
[T] : \hat{y}_t \neq y_t\}|$. Assume a hypothesis class $\ch \subseteq \cy^\cx$.
In the realizable setting, we assume that for some function $f \in \ch$ all the outcomes
are evaluations of $f$, namely, $y_t = f(x_t)$.

Learning in the realizable online model has been studied by \cite{Littlestone87},
who showed that a combinatorial measure, called the Littlestone
dimension, characterizes the min-max optimal number of mistakes for {\em binary}
hypotheses classes in the realizable case. We propose a generalization
of the Littlestone dimension to multiclass hypotheses classes.

Consider a rooted tree $T$ whose internal nodes are labeled by
elements from $\cx$ and whose edges are labeled by elements from $\cy$, such that
the edges from a single parent to its child-nodes are each labeled with a different label. The tree $T$ is \emph{shattered} by $\ch$ if, for
every path from root to leaf which traverses the nodes $x_1,\ldots,x_k$, there is a function
$f\in\ch$ such that $f(x_i)$ is the label of the edge
$(x_i,x_{i+1})$.  We define the \emph{Littlestone dimension} of a multiclass hypothesis class $\ch$,
denoted $\ldim(\ch)$, to be the maximal depth of a complete binary tree that is shattered by $\ch$ (or $\infty$ if there are a shattered trees for arbitrarily large depth).

As we presently show, the number $\ldim(\ch)$ fully characterizes the worst-case mistake bound for the online model in the realizable setting. The upper bound is achieved using the following algorithm.
\begin{tabbing}
{\bf Algorithm:} Standard Optimal Algorithm (SOA)\\
Initialization: $V_0=\ch$.\\
For \=$t=1,2\ldots$,\\
\>receive $x_t$\\
\>for $y\in \cy$, let $V_t^{(y)}=\{f\in V_{t-1}:f(x_t)=y\}$\\
\>predict $\hat{y}_t\in\arg \max_y\ldim(V_t^{(y)})$\\
\>receive true answer $y_t$\\
\>update $V_t=V_t^{(y_t)}$
\end{tabbing}

\begin{theorem}
The $SOA$ algorithm makes at most $\ldim(\ch)$ mistakes on any realizable sequence. Furthermore, the worst-case number of mistakes of \emph{any} deterministic online algorithm is at least $\ldim(\ch)$.
For any randomized online algorithm, the expected number of mistakes on the worst sequence is at least $\frac12 \ldim(\ch)$.
\end{theorem}

\begin{proof} (sketch)
First, we show that the $SOA$ algorithm makes at most $\ldim(\ch)$ mistakes.
The proof is a simple adaptation of the proof of the binary case
\citep[see][]{Littlestone87,Shalev12}. We note that for each $t$ there is
at most one $y\in \cy$ with
$\ldim(V_t^{(y)})=\ldim(V_t)$, and for the rest
of the labels we have
$\ldim(V_t^{(y)})<\ldim(V_t)$ (otherwise, it is not hard to construct a tree of depth $\ldim(V_t)+1$, whose root is $x_t$,  that is shattered by $V_t$). Thus, whenever the
algorithm errs, the Littlestone dimension of $V_t$ decreases by at least $1$,
so after $\ldim(\ch)$ mistakes, $V_t$ is composed of a single
function. 

For the second part of the theorem, it is not hard to see that, given a shattered tree of depth $\ldim(\ch)$, the environment can force any deterministic online learning algorithm to make $\ldim(\ch)$ mistakes. Note also that allowing the algorithm to make randomized predictions cannot be too helpful. It is easy to see that given a shattered tree of depth $\ldim(\ch)$, the environment can enforce any randomized online learning algorithm to make at
least $\ldim(\ch)/2$ mistakes on average, by traversing the shattered tree, and providing at every round the label that the randomized algorithm is less likely to predict.
\end{proof}

In the agnostic case, the sequence of outcomes, $y_1,\ldots,y_m$, is not necessarily consistent with some function $f \in \ch$. Thus, one wishes to bound the \emph{regret} of the algorithm, instead of its absolute number of mistakes. The regret is the difference between the number of mistakes made by the algorithm and the number of mistakes made by the best-matching function $f \in \ch$. The agnostic case for classes of binary-output functions has been studied in \cite{Ben-DavidPaSh09}. It was shown that, as in the realizable case, the
Littlestone dimension characterizes the optimal regret bound.

We show that the generalized Littlestone dimension characterizes the optimal regret bound for the multiclass case as well. The proof follows the paradigm of `learning with expert advice' \citep[see e.g.][]{CesaLu06,Shalev12}, which we now briefly
describe. Suppose that at each step, $t$, before the algorithm chooses
its prediction, it observes $N$ {\em advices} $(f_1^t,\ldots,f_N^t)\in
\cy^N$, which can be used to determine its prediction. We think of
$f_i^t$ as the prediction made by the {\em expert} $i$ at time $t$ and
denote the {\em loss}
of the expert $i$ at time $T$ by $L_{i,T}=|\{t\in [T]:f_{i,t}\ne y_{t}\}|$ . The goal here it to devise an algorithm
that achieves a loss which is comparable with the loss of the best
expert. Given $T$, the following algorithm \citep[chapter 2]{CesaLu06}
achieves expected loss at most $\min_{i\in [N]}L_{i,T}+\sqrt{\frac{1}{2}\ln(N)T}$.
\begin{tabbing}
{\bf Algorithm:} Learning with Expert Advice (LEA)\\
{\bf Parameters:} Time horizon -- T\\
Set $\eta=\sqrt{8\ln(N)/T}$\\
For \=$t=1,2\ldots,T$\\
\>receive expert advices $(f_1^t,\ldots, f_N^t)\in \cy^N$\\
\>predict $\hat{y}_t = f_{i,t}$ with probability proportional to $\exp(-\eta L_{i,{t-1}})$\\
\>receive true answer $y_t$
\end{tabbing}

We use this algorithm and its guarantee to prove the following theorem.
\begin{theorem}\label{thm:online_agnostic}
In the agnostic online multiclass setting, the expected loss of the optimal algorithm on the worst-case sequence
is at most $\min_{f\in \ch}L_{f,T}+\sqrt{\frac{1}{2}\ldim(\ch)T\log(Tk)}$.
\end{theorem}
\begin{proof}
In $\ldim(\ch)=1$, then $|\ch|=1$ and the theorem is clear. We can therefore assume that $\ldim(\ch)\ge 2$.
First, construct an expert for every $f\in\ch$, whose advice at time $t$ is $f(x_t)$. Denote the loss of the expert corresponding to $f$ at time $t$ by $L_{f,t}$. Running the algorithm LEA with this set of experts yields an algorithm whose expected error is at most $\min_{f\in \ch}L_{f,T}+\sqrt{\frac{1}{2}\ln(|\ch|)T}$. Our goal now is to construct a more compact set of experts, which will allow us to bound the loss in terms of $\ldim(\ch)$ instead of $\ln(|\ch|)$.

Given time horizon $T$, let $A_T = \{A\subset [T]\mid |A|\le\ldim(\ch)\}$. For every $A\in A_T$ and $\phi:A\to \cy$, we define an expert $E_{A,\phi}$. The expert $E_{A,\phi}$ imitates the SOA algorithm when it errs exactly on the examples $\{ x_t \mid t\in A\}$ and the true labels of these examples are determined by $\phi$. Formally, the expert $E_{A,\phi}$ proceeds as follows:
\begin{tabbing}
Set $V_1=\ch$.\\
For \=$t=1,2\ldots,T$\\
\>Receive $x_t$.\\
\>Set $l_t=\argmax_{y \in \cy}\ldim(\{f\in V_t:f(x_t)=y\})$. \\
\>If $t\in A$, Predict $\phi(t)$ and update  $V_{t+1}=\{f\in V_t:f(x_t)=\phi(t)\}$.\\
\> If $t\not\in A$, Predict $l_t$ and update  $V_{t+1}=\{f\in V_t:f(x_t)=l_t\}$.\\
\end{tabbing}
The number of experts we constructed is 
$ \sum_{j=0}^{\ldim(\ch)}\binom{T}{j}k^{j}\le (Tk)^{\ldim(\ch)}$. Denote the number of mistakes made by the expert $E_{A,\phi}$ after $T$ rounds by $L_{A,\phi,T}$. If we apply the LEA algorithm with the set of experts we have constructed, the resulting algorithm makes at most
$$\min_{A,\phi}L_{A,\phi,T}+\sqrt{\frac{1}{2}T\ldim(\ch)\ln(Tk)}$$
mistakes. We claim that $\min_{A,\phi}L_{A,\phi,T}\le \min_{f\in\ch}L_{f,T}$: Let $f\in \ch$. 
Denote by $A\subset [T]$ the set of rounds in which the SOA algorithm errs when running on the sequence $(x_1,f(x_1)),\ldots,(x_T,f(x_T))$ and define $\phi :A\to \cy$ by $\phi(t)=f(x_t)$. Since the SOA algorithm makes at most $\ldim(\ch)$ mistakes, $|A|\le \ldim(\ch)$. It is not hard to see that the predictions of the expert $E_{A,\phi}$ coincide with the predictions of the expert $E_f$. Thus, $L_{A,\phi,T}=L_{f,T}$. 
\end{proof}

Adapting the proof of Lemma 14 from \cite{Ben-DavidPaSh09}, we conclude a corresponding lower bound:
\begin{theorem}
In the agnostic online multiclass setting, the expected loss of every algorithm on the worst-case sequence
is at least $\min_{f\in \ch}L_{f,T}+\sqrt{\frac{1}{8}\ldim(\ch)T}$.
\end{theorem}
We leave as an open question to close the gap between the bounds in the above Theorems. Note that this gap is analogous to the sample complexity gap for ERM learners in the PAC setting, seen in \thmref{th:multiclaas-bounds}.

\subsection{The Bandit Setting}\label{ch:bandits}
So far we have assumed that the label of each training example is fully revealed. 
In this section we deal with the
bandit setting. In this setting, the learner does not
get to see the correct label of a training example. Instead, the
learner first receives an instance $x \in \cx$, and should guess a
label, $\hat{y}$. The learner then receives a binary response,
which indicates only whether the guess was correct or not. If the guess is
correct then the learner knows the identity of the correct label. If the guess is wrong, the learner only knows that $\hat{y}$ is not the correct label, and not the identity of the correct label.

\subsubsection{Bandit vs. Full Information in the Batch Model}
In this section we consider the bandit setting in the batch model. In this setting the sample is drawn i.i.d. 
as before, but the learner first observes only the instances $x_1,\ldots,x_m$. The learner then
guesses a label for each of the instances, and receives a binary response indicating for each label whether it was the correct one. 

Let $\ch\subseteq\cy^\cx$ be a hypothesis class and let $k=|\cy|$.
Our goal is to analyze the \emph{realizable bandit sample complexity} of $\ch$, which we denote by $m_{b}^{r}(\epsilon,\delta)$, and the \emph{agnostic bandit sample complexity} of $\ch$, which we denote by $m_{b}^{a}(\epsilon,\delta)$. The following theorem provides upper bounds on the sample complexities.

\begin{theorem}\label{th:bandit-PAC}
Let $\ch \subseteq \cy^\cx$ be a hypothesis class. Then,
\[
m^{r}_b(\epsilon,\delta)=O\left( k\cdot \frac{d_G(\ch)\cdot  \ln\left(\frac{1}{\epsilon}\right)+\ln(\frac{1}{\delta})}{\epsilon}\right) \text{ and }
m^{a}_b(\epsilon,\delta)=O\left( k\cdot \frac{d_G(\ch)+\ln(\frac{1}{\delta})}{\epsilon^2}\right) ~.
\]
\end{theorem}
\begin{proof}
   Let $\ca_f$ be a (full information)
  ERM learner for $\ch$. Consider the following algorithm, denoted $\ca_b$, for the
  bandit setting: Given a sample $(x_i,y_i)_{i=1}^m$, for each $i$ the
  algorithm guesses a label $\hat y_i\in\cy$ drawn uniformly at
  random. Then the algorithm calls 
  $\ca_f$ with an input sample which consists only of the sample pairs
  for which the binary response indicated that the guess $\hat{y_i}$ was correct. 
  Thus, the input sample is $\{(x_i,\hat{y}_i) \mid \hat{y}_i = y_i\}$.
  $\ca_b$ then returns whatever hypothesis $\ca_f$ returned.
  
  We show that $m^r_{\ca_b}(\epsilon,\delta)\le 3k\cdot m^r_{\ca_f}(\epsilon,\frac{\delta}{2})+\frac{3}{2}\log\left(\frac{2}{\delta}\right)=:m'$
  and similarly for the agnostic case, so that the theorem is implied by the
  bounds in the full information setting (\thmref{th:ERM-Bounds}).  Indeed, suppose that
  $m$ examples suffice for $\ca_f$ to return a hypothesis
  with excess error at most $\epsilon$, with probability at least
  $1-\frac{\delta}{2}$. Let $(x_i,y_i)_{i=1}^{m'}$ be a sample for
  the bandit algorithm. By Chernoff's bound, with probability at least
  $1-\frac{\delta}{2}$, $\ca_b$ guesses correctly the label of at least $m$ examples.
  Therefore $\ca_f$ runs on a sample of at least this size. The
  sample that $\ca_f$ receives is a conditionally i.i.d. sample, given the size of the sample, with the same conditional distribution as the one
  the original sample was sampled from. Thus, with probability at least $1-\frac{\delta}{2}$,
  $\ca_f$ (and, consequently, $\ca_b$)
  returns a hypothesis with excess error at most $\epsilon$.
\end{proof}

An interesting quantity to consider is the price of bandit information in the batch model: Let $\ch$ be a hypotheses class, and define $\pbi_\ch(\epsilon,\delta)=m^{r}_{b,\ch}(\epsilon,\delta)/m^r_{\pac,\ch}(\epsilon,\delta)$. By Theorems \ref{th:bandit-PAC} and \ref{th:multiclaas-bounds} and Equation \ref{eq:ben-david} we see that, $\pbi(\epsilon,\delta)=O(\ln(\frac{1}{\epsilon})k\ln(k))$. This is essentially tight since it is not hard to see that if both $\cx,\cy$ are finite and we let $\ch=\cy^{\cx}$, then $\pbi_\ch=\Omega(k)$.

Using Theorems  \ref{th:bandit-PAC} and \ref{th:multiclaas-simple-bounds} and
Equation \ref{eq:ben-david}  we can further conclude that, as in
the full information case, the finiteness of the Natarajan dimension is
necessary and sufficient for learnability in the bandit setting as
well. However, the ratio between the upper bound due to \thmref{th:bandit-PAC} and the lower bound, due to 
\thmref{th:multiclaas-simple-bounds}, is $\Omega(\ln(k)\cdot k)$. It would be interesting to find a more tight
characterization of the sample complexity in the bandit setting.  This characterization cannot depend solely on the Natarajan dimension, or other quantities which are strongly related to it (such as the graph dimension or other notion of dimension defined in \citet{Ben-DavidCeHaLo95}):
For example, the classes $[k]^{[d]}$ and $[2]^{[d]}$ have the same Natarajan dimension, but their bandit sample complexity differs by a factor of $\Omega(k)$.

\subsubsection{Bandit vs. Full Information in the Online Model}
We now consider Bandits in the online learning model. We focus on the
realizable case, in which the feedback provided to the learner is
consistent with some function $f_0 \in \ch$. We define a new notion of
dimension of a class, that determines the sample complexity in this
setting. 

As in Section \ref{ch:online}, consider a rooted tree $T$ whose internal nodes are labeled by
elements from $\cx$ and whose edges are labeled by elements from $\cy$, such that
the edges from a single parent to its child-nodes are each labeled with a different label. The tree $T$ is \emph{BL-shattered} by $\ch$ if, for every path from root to leaf $x_1,\ldots,x_k$, there is a function $f\in\ch$ such that for every $i$, $f(x_i)$ is \emph{different} from the label of $(x_i,x_{i+1})$.
The {\bf Bandit-Littlestone dimension} of $\ch$, denoted $\operatorname{BL-dim}(\ch)$, is the maximal depth of a complete $k$-ary tree that is BL-shattered by $\ch$.

\begin{theorem}
Let $\ch$ be a hypothesis class with $L=\bldim(\ch)$. Then every deterministic online bandit learning algorithm for $\ch$ will make at least $L$ mistakes in the worst case. Moreover, there is an online learning algorithm that makes at most $L$ mistakes on every realizable sequence.
\end{theorem}
\begin{proof}
First, let $T$ be a BL-shattered tree of depth $L$. We show that for
every deterministic learning algorithm there is a sequence $x_1,\ldots,x_L$ and
a labeling function $f_0 \in \ch$ such that the algorithm makes $L$ mistakes on
this sequence. The sequence consists of the instances attached to nodes of $T$,
when traversing the tree from the root to one of its leaves, such that the
label of each edge $(x_i,x_{i+1})$ is equal to the algorithm's prediction
$\hat{y}_i$. The labeling function $f_0 \in \ch$ is one such that for all $i$,
$f_0(x_i)$ is different from the label of edge $(x_i,x_{i+1})$. Such a function
exists since $T$ is BL-shattered, and the algorithm will clearly make $L$ mistakes
on this sequence.

Second, the following online learning algorithm makes at most $L$ mistakes on any realizable input sequence.
\begin{tabbing}
{\bf Algorithm:} Bandit Standard Optimal Algorithm (BSOA)\\
Initialization: $V_0=\ch$.\\
For \=$t=1,2\ldots$,\\
\>Receive $x_t$\\
\>For $y\in \cy$, let $V_t^{(y)}=\{f\in V_{t-1}:f(x_t)\ne y\}$\\
\>Predict $\hat{y}_t\in\arg \min_y\bldim(V_t^{(y)})$\\
\>Receive an indication whether $\hat{y}_t=f(x_t)$\\
\>If the prediction is wrong, update $V_t=V_t^{(\hat{y}_t)}.$
\end{tabbing}
To see that BSOA makes at most $L$ mistakes, note that at each time $t$, there is at least one $V_t^{(y)}$ with $\bldim(V_t^{(y)})<\bldim(V_{t-1})$. This can be seen by assuming to the contrary that this is not so, and concluding that if $\bldim(V_t^{(y)}) = \bldim(V_{t-1})$ for all $y \in [k]$, then
one can construct a shattered tree of size $\bldim(V_{t-1}) + 1$ for $V_{t-1}$, thus reaching a contradiction.

Thus, whenever the algorithm errs, the dimension of $V_t$ decreases by one. Thus, after $L$ mistakes, the dimension is $0$, which means that there is a single function that is consistent with the sample, so no more mistakes can occur.
\end{proof}

{\bf The price of bandit information:}
Let $\pbi(\ch)=\bldim(\ch)/\ldim(\ch)$ and fix $k\ge 2$. How large can $\pbi(\ch)$ be when $\ch$ is a class of functions from a domain $\cx$ to a range $\cy$ of cardinality $k$? We refer the reader to~\cite{daniely2013price}, where it is shown that $\pbi(\ch)\le 4k\log(k)$. This bound is tight up to the logarithmic factor.

\section{Discussion}

We have shown in this work that even in the simple case of multiclass learning, different ERM learners for the same problem can have large gaps in their sample complexities. To put our results in a more general perspective, consider the {\em General Setting of Learning} introduced by
\cite{Vapnik98}. In this setting, a {\em learning problem} is a
triplet $(\ch,\cz,l)$, where $\ch$ is a hypothesis class, $\cz$ is a
data domain, and $l : \ch \times \cz \to \reals$ is a loss
function. We emphasize that $\mathcal H$ is not necessarily a class of
functions but rather an abstract set of models. The goal of the
learner is, given a sample $S\in\mathcal{Z}^m$, sampled from some
(unknown) distribution $\mathcal{D}$ over $\mathcal Z$, to find a
hypothesis $h\in\mathcal{H}$ that minimizes the {\em expected loss},
$l(h)= \mathbb{E}_{z\sim\mathcal D}[l(h,z)]$.

The general setting of learning encompasses multiclass learning as
follows: given a hypotheses class $\mathcal{H}\subset
\mathcal{Y}^\mathcal{X}$, take
$\mathcal{Z}=\mathcal{X}\times\mathcal{Y}$ and define
$l:\mathcal{H}\times\mathcal{Z}\to\mathbb{R}$ by $l(h,(x,y))=1[h(x)\ne
y]$. However, the general learning setting encompasses many other
problems as well, for instance:
\begin{itemize}
\item {\bf Regression with the squared loss:} Here, $\mathcal{Z}=\mathbb{R}^n\times\mathbb{R}$,  $\mathcal{H}$ is a set of real-valued functions over $\mathbb R^n$ and $l(h,(x,y))=(h(x)-y)^2$.

\item {\bf k-means:} Here, $\mathcal{Z}=\mathbb{R}^n$,  $\mathcal{H}=(\mathbb{R}^n)^k$ and, for $h=(c_1,\ldots,c_k)\in \mathcal H$ and $\;x\in \mathcal Z$, the loss is $l((c_1,\ldots,c_k),x)=\min_{j\in[k]} ||c_j-x||^2$.

\item {\bf Density estimation:} Here, $\mathcal{Z}$ is an arbitrary
  finite set, $\mathcal{H}$ is some set of probability density
  functions over $\cz$, and the loss function is the log loss, $l(p,x)=-\ln(p(x))$.
\end{itemize}
A learning problem is {\em learnable} in the general setting of learning if there exists a function $\mathcal{A}:\cup_{m=1}^\infty \mathcal{Z}^m\to\mathcal H$ such that for every $\epsilon>0$ and $\delta>0$ there exists an $m$ such that for every distribution $\mathcal{D}$ over $\mathcal Z$, 
$$\Pr_{S\sim \mathcal{Z}^m}\left(l(\mathcal{A}(S))\ge \inf_{h\in\mathcal H}l(h)+\epsilon\right)<\delta$$
A learning problem {\em converges uniformly} if, for every $\epsilon >0$,
$$\lim_{m\to \infty}\Pr_{S\sim \mathcal{Z}^m}\left(\sup_{h\in\mathcal H}|l(h)-l_S(h)|>\epsilon\right)= 0$$
where for $S=(z_1,\ldots,z_m)\in \mathcal Z^m$,
$l_S(h)=\frac{1}{m}\sum_{i=1}^ml(h,z_i)$ is the empirical loss of $h$
on the sample $S$. An easy observation is that uniform convergence
implies learnability, and a classical result is that for binary
classification and for regression (with absolute or squared loss), the
inverse implication also holds. Thus, it was believed that excluding
some trivialities, learnability is equivalent to uniform
convergence. In \cite{ShalevShSrSr10} it is shown that for stochastic
convex optimization, learnability does not imply uniform convergence,
giving an evidence that the above belief might be misleading. Our
results in this work can be seen as another step in this direction, as we have
shown that even in multiclass classification -- a simple, natural and
popular generalization of binary classification, the above mentioned
equivalence no longer holds.

We conclude with an open question. In view of our results in \secref{sec:pac}, the following conjecture suggests itself.
\begin{conjecture}\label{conj:main}
There exists a constant $C$ such that, for \emph{every} hypothesis class $\ch\subseteq \cy^\cx$,
\[
m^r_\pac(\epsilon,\delta)
\le C\left(\frac{d_N(\ch)\ln(\frac{1}{\epsilon})+\ln(\frac{1}{\delta})}{\epsilon}\right)
\]
\end{conjecture}
In light of Theorem \ref{th:ERM-Bounds} and the fact that there are cases where $d_G\ge\log_2(k-1)d_N$, the conjecture can only be proved if this learning rate can be achieved by a learning algorithm that is not just an {\em arbitrary} ERM learner. So far, all the general upper bounds that we are aware of are valid for \emph{any} ERM learner. Understanding how to select among ERM learners is fundamental as it teaches us what is the optimal way to learn. We hope that our examples from section \ref{sec:pac} and our result for symmetric classes will lead to a better understanding of the optimal learning method.

\paragraph{Acknowledgments:}
We thank Ohad Shamir for valuable comments.
Shai Shalev-Shwartz is supported by the Israeli Science Foundation grant number 598-10. Amit Daniely is a recipient of the Google Europe Fellowship in Learning
Theory, and this research is supported in part by this Google Fellowship

\bibliographystyle{plainnat}
\bibliography{curRefs}

\begin{thebibliography}{24}
\providecommand{\natexlab}[1]{#1}
\providecommand{\url}[1]{\texttt{#1}}
\expandafter\ifx\csname urlstyle\endcsname\relax
  \providecommand{\doi}[1]{doi: #1}\else
  \providecommand{\doi}{doi: \begingroup \urlstyle{rm}\Url}\fi

\bibitem[Allwein et~al.(2000)Allwein, Schapire, and Singer]{AllweinScSi00a}
E.~L. Allwein, R.E. Schapire, and Y.~Singer.
\newblock Reducing multiclass to binary: A unifying approach for margin
  classifiers.
\newblock \emph{Journal of Machine Learning Research}, 1:\penalty0 113--141,
  2000.

\bibitem[Alon et~al.(1997)Alon, Ben-David, Cesa-Bianchi, and
  Haussler]{AlonBeCsHa97}
N.~Alon, S.~Ben-David, N.~Cesa-Bianchi, and D.~Haussler.
\newblock Scale-sensitive dimensions, uniform convergence, and learnability.
\newblock \emph{Journal of the ACM (JACM)}, 44\penalty0 (4):\penalty0 615--631,
  1997.

\bibitem[Anthony and Bartlett(1999)]{AB}
M.~Anthony and P.~L. Bartlett.
\newblock \emph{Neural Network Learning: Theoretical Foundations}.
\newblock Cambirdge University Press, 1999.

\bibitem[Auer et~al.(2002)Auer, Cesa-Bianchi, and Fischer]{AuerCeFi02}
P.~Auer, N.~Cesa-Bianchi, and P.~Fischer.
\newblock {Finite-time analysis of the multiarmed bandit problem}.
\newblock \emph{Machine learning}, 47\penalty0 (2):\penalty0 235--256, 2002.

\bibitem[Auer et~al.(2003)Auer, Cesa-Bianchi, Freund, and
  Schapire]{AuerCeFrSc03}
P.~Auer, N.~Cesa-Bianchi, Y.~Freund, and R.E. Schapire.
\newblock The nonstochastic multiarmed bandit problem.
\newblock \emph{SICOMP: SIAM Journal on Computing}, 32, 2003.

\bibitem[Bartlett and Mendelson(2002)]{BartlettMe02}
P.~L. Bartlett and S.~Mendelson.
\newblock Rademacher and {G}aussian complexities: {R}isk bounds and structural
  results.
\newblock \emph{Journal of Machine Learning Research}, 3:\penalty0 463--482,
  2002.

\bibitem[Bartlett et~al.(1996)Bartlett, Long, and Williamson]{BartlettLoWi1996}
PL~Bartlett, PM~Long, and RC~Williamson.
\newblock Fat-shattering and the learnability of real-valued functions.
\newblock \emph{Journal of Computer and System Sciences}, 52\penalty0
  (3):\penalty0 434--452, 1996.

\bibitem[Ben-David et~al.(1995)Ben-David, Cesa-Bianchi, Haussler, and
  Long]{Ben-DavidCeHaLo95}
S.~Ben-David, N.~Cesa-Bianchi, D.~Haussler, and P.~Long.
\newblock Characterizations of learnability for classes of
  $\{0,\ldots,n\}$-valued functions.
\newblock \emph{Journal of Computer and System Sciences}, 50:\penalty0 74--86,
  1995.

\bibitem[Ben-David et~al.(2009)Ben-David, Pal, , and
  Shalev-Shwartz]{Ben-DavidPaSh09}
S.~Ben-David, D.~Pal, , and S.~Shalev-Shwartz.
\newblock Agnostic online learning.
\newblock In \emph{COLT}, 2009.

\bibitem[Beygelzimer et~al.(2007)Beygelzimer, Langford, and
  Ravikumar]{BeygelzimerLaRa07}
A.~Beygelzimer, J.~Langford, and P.~Ravikumar.
\newblock Multiclass classification with filter trees.
\newblock \emph{Preprint, June}, 2007.

\bibitem[Cesa-Bianchi and Lugosi(2006)]{CesaLu06}
Nicolo Cesa-Bianchi and Gabor Lugosi.
\newblock \emph{Prediction, Learning, and Games}.
\newblock Cambridge University Press, 2006.

\bibitem[Daniely and Helbertal(2013)]{daniely2013price}
Amit Daniely and Tom Helbertal.
\newblock The price of bandit information in multiclass online classification.
\newblock In \emph{Conference on Learning Theory}, pages 93--104, 2013.

\bibitem[Dietterich and Bakiri(1995)]{DietterichBa95}
T.~G. Dietterich and G.~Bakiri.
\newblock Solving multiclass learning problems via error-correcting output
  codes.
\newblock \emph{Journal of Artificial Intelligence Research}, 2:\penalty0
  263--286, January 1995.

\bibitem[Hoeffding(1963)]{Hoeffding63}
W.~Hoeffding.
\newblock Probability inequalities for sums of bounded random variables.
\newblock \emph{Journal of the American Statistical Association}, 58\penalty0
  (301):\penalty0 13--30, March 1963.

\bibitem[Kakade et~al.(2008)Kakade, Shalev-Shwartz, and Tewari]{KakadeShTe08}
S.M. Kakade, S.~Shalev-Shwartz, and A.~Tewari.
\newblock Efficient bandit algorithms for online multiclass prediction.
\newblock In \emph{International Conference on Machine Learning}, 2008.

\bibitem[Kearns et~al.(1994)Kearns, Schapire, and Sellie]{KearnsScSe94}
Michael~J. Kearns, Robert~E. Schapire, and Linda~M. Sellie.
\newblock Toward efficient agnostic learning.
\newblock \emph{Machine Learning}, 17:\penalty0 115--141, 1994.

\bibitem[Littlestone(1987)]{Littlestone87}
N.~Littlestone.
\newblock Learning when irrelevant attributes abound.
\newblock In \emph{FOCS}, pages 68--77, October 1987.

\bibitem[Natarajan(1989)]{Natarajan89b}
B.~K. Natarajan.
\newblock On learning sets and functions.
\newblock \emph{Mach. Learn.}, 4:\penalty0 67--97, 1989.

\bibitem[Shalev-Shwartz(2012)]{Shalev12}
S.~Shalev-Shwartz.
\newblock Online learning and online convex optimization.
\newblock \emph{Foundations and Trends in Machine Learning}, 4\penalty0
  (2):\penalty0 107--194, 2012.

\bibitem[Shalev-Shwartz et~al.(2010)Shalev-Shwartz, Shamir, Srebro, and
  Sridharan]{ShalevShSrSr10}
S.~Shalev-Shwartz, O.~Shamir, N.~Srebro, and K.~Sridharan.
\newblock Learnability, stability and uniform convergence.
\newblock \emph{The Journal of Machine Learning Research}, 9999:\penalty0
  2635--2670, 2010.

\bibitem[Valiant(1984)]{Valiant84}
L.~G. Valiant.
\newblock A theory of the learnable.
\newblock \emph{Communications of the ACM}, 27\penalty0 (11):\penalty0
  1134--1142, November 1984.

\bibitem[Vapnik(1998)]{Vapnik98}
V.~N. Vapnik.
\newblock \emph{Statistical Learning Theory}.
\newblock Wiley, 1998.

\bibitem[Vapnik and Chervonenkis(1971)]{VapnikCh71}
V.~N. Vapnik and A.~Ya. Chervonenkis.
\newblock On the uniform convergence of relative frequencies of events to their
  probabilities.
\newblock \emph{Theory of Probability and its applications}, XVI\penalty0
  (2):\penalty0 264--280, 1971.

\bibitem[Vapnik(1995)]{Vapnik95}
V.N. Vapnik.
\newblock \emph{The Nature of Statistical Learning Theory}.
\newblock Springer, 1995.

\end{thebibliography}

\end{document}